%% file: equiv_and_implicit_arxiv.tex
\documentclass{article}
\usepackage[round,comma]{natbib}
\usepackage{amssymb}
\usepackage{amsmath}
\usepackage{amsthm}
\usepackage{xspace}
\usepackage{anysize}
\usepackage{comment}
\usepackage{color}
\usepackage{float}
\usepackage{enumerate}
\usepackage[pdftex]{graphicx}

\marginsize{3cm}{4cm}{3cm}{4cm}
\usepackage{url}

\newtheorem{theorem}{Theorem}
\newtheorem{lemma}[theorem]{Lemma}

\newtheorem{corollary}[theorem]{Corollary}
\newcommand{\BlackBox}{\rule{1.5ex}{1.5ex}}  %

\newenvironment{keywords}
{\bgroup\leftskip 20pt\rightskip 20pt \small\noindent{\bf Keywords:} }%
{\par\egroup\vskip 0.25ex}

\input{equiv_and_implicit_includes}

\begin{document}

\title{A Unified View of Regularized Dual Averaging \\
and Mirror Descent with Implicit Updates
}
\author{H. Brendan McMahan \\
Google, Inc. \\
\texttt{\small{mcmahan@google.com}}
}
\maketitle

\input{equiv_and_implicit_contents}

%% file: equiv_and_implicit_includes.tex
\newenvironment{myproof}[1]%
{\par\noindent{\bf #1\ }}%
{\hfill\BlackBox\\[2mm]}

\newcommand{\emcite}{\citet}
\renewcommand{\cite}{\citep}
\newcommand{\ifrac}[2]{#1/#2}  %

\newcommand{\hf}{f^R}
\newcommand{\h}{\frac{1}{2}}

\newcommand {\norm}[1]{\big\| #1 \big\|}
\newcommand {\smnorm}[1]{\| #1 \|}

\newcommand{\hnorm}[1]{\frac{1}{2}\norm{#1}}

\newcommand{\argmin}{\mathop{\rm arg\,min}}
\newcommand{\BO}{\mathcal{O}}

\newcommand {\set}[1]{\ensuremath{\left\{#1\right\}}}

\newcommand{\R}{\ensuremath{\mathbb{R}}}

\newcommand{\grad}{\triangledown}

\newcommand{\tp}{\top}  %

\newcommand{\invh}{^{-\h}}
\newcommand{\inv}{^{-1}}
\newcommand{\ti}{_{t+1}}

\newcommand{\Qa}{Q_a}
\newcommand{\Qb}{Q_b}
\newcommand{\Qtt}{Q_{1:t}}

\newcommand{\Rp}{\tilde{R}}
\newcommand{\Rtt}{R_{1:t}}

\newcommand{\gtt}{g_{1:t}}

\newcommand{\hx}{\hat{x}}
\newcommand{\rtt}{r_{1:t}}

\newcommand{\xs}{\mathring{x}} %
\newcommand{\xti}{x_{t+1}}
\newcommand{\xt}{x_t}
\newcommand{\Fs}{\mathcal{F}}
\newcommand{\IF}{I_\Fs}  %

\newcommand{\Regret}{\text{Regret}}
\newcommand{\Snp}{S^n_{+}}
\newcommand{\Snpp}{S^n_{++}}

\newcommand{\gpt}{g'_t}  %

\newcommand{\fln}{\bar{f}} %

\newcommand{\mge}{\succeq}

\newcommand{\Qab}{Q_{a:b}}
\newcommand{\Qah}{\Qa^\h}
\newcommand{\Qbh}{\Qb^\h}
\newcommand{\Qabh}{\Qab^\h}

\newcommand{\yc}{y}  %

\newcommand{\prox}{FTRL-Proximal\xspace}
\newcommand{\orgn}{RDA\xspace}
\newcommand{\fobos}{FOBOS\xspace}
\newcommand{\bs}{\bar{\sigma}}

\newcommand{\Brtt}{\Br_{1:t}}

\newcommand{\Br}{\mathcal{B}}  %

\newcommand{\hp}{\hat{\Psi}}

\newcommand{\hha}{h_1}
\newcommand{\hhb}{h_2}

\newcommand{\xa}{x_1}
\newcommand{\xb}{x_2}

\newcommand{\hpa}{\Phi_a}
\newcommand{\hpb}{\Phi_b}

\newcommand{\pxa}{x_1}
\newcommand{\pxb}{x_2}

\newcommand{\wh}{\tilde{h}}
\newcommand{\wx}{\tilde{x}}
\newcommand{\wg}{\tilde{g}}

\newcommand{\lh}{\bar{h}}
\newcommand{\lx}{\bar{x}}

%% file: equiv_and_implicit_contents.tex
\begin{abstract}
  We study three families of online convex optimization algorithms:
  follow-the-proximally-regularized-leader (FTRL-Proximal),
  regularized dual averaging (RDA), and composite-objective mirror
  descent.  We first prove equivalence theorems that show all of these
  algorithms are instantiations of a general FTRL update.  This
  provides theoretical insight on previous experimental observations.
  In particular, even though the FOBOS composite mirror descent
  algorithm handles $L_1$ regularization explicitly, it has been
  observed that RDA is even more effective at producing sparsity.  Our
  results demonstrate that FOBOS uses subgradient approximations to
  the $L_1$ penalty from previous rounds, leading to less sparsity
  than RDA, which handles the cumulative penalty in closed form.  The
  FTRL-Proximal algorithm can be seen as a hybrid of these two, and
  outperforms both on a large, real-world dataset.
  
  Our second contribution is a unified analysis which produces regret
  bounds that match (up to logarithmic terms) or improve the best
  previously known bounds.  This analysis also extends these
  algorithms in two important ways: we support a more general type of
  composite objective and we analyze implicit updates, which replace
  the subgradient approximation of the current loss function with an
  exact optimization.
\end{abstract}

\vspace{0.1in}

\begin{keywords}
online learning, online convex optimization, subgradient methods, regret bounds, follow-the-leader algorithms
\end{keywords}

\section{Introduction}
We consider the problem of online convex optimization, and in
particular its application to online learning.  On each round
$t=1,\dots, T$, we must pick a point $\xt \in \R^n$.  A convex loss
function $f_t$ is then revealed, and we incur loss $f_t(\xt)$.  Our
regret at the end of $T$ rounds with respect to a comparator point
$\xs$ is
\begin{equation*}
\Regret \equiv \sum_{t=1}^T f_t (\xt) - \sum_{t=1}^T f_t(\xs).
\end{equation*}

In Section~\ref{sec:analysis} we provide a unified regret analysis of
three prominent algorithms for online convex optimization.
In recent years, these algorithms have received significant attention
because they have straightforward and efficient implementations and
offer state-of-the-art performance for many large-scale applications.
In particular, we consider:
\begin{itemize}
\item Follow-the-Proximally-Regularized-Leader (FTPRL), introduced
  with adaptive learning rates (regularization)
  by~\emcite{mcmahan10boundopt}.
\item Regularized Dual Averaging (RDA), introduced
  by~\emcite{xiao09dualaveraging} and extended with adaptive learning
  rates by~\emcite{duchi10adaptive}.
\item Composite-Objective Mirror Descent (COMID)
  algorithms~\cite{duchi10composite}, including
  FOBOS~\cite{duchi09fobos}.
\end{itemize}

As pointed out by~\emcite{duchi10composite}, the analyses of RDA and
COMID cited above are completely different.  In contrast, we provide a
unified analysis of these algorithms.
One of our contributions is simply demonstrating that this large and
important family of algorithms can be analyzed using a common
argument, but our analysis also generalizes previous results in
several important ways.
First, we extend all of these algorithm to handle implicit updates,
which replace the first-order approximation on the current loss
function with an exact optimization.  In many practical situations
this update can be solved efficiently, and offers both theoretical and
practical benefits compared to the first-order update.

We also extend the ability of these algorithms to handle composite
objectives (objectives that include a fixed non-smooth term $\Psi$).
Previous work considers loss functions on each round of the form
$f_t(x) + \Psi(x)$, where $f_t$ is approximated by a linear function,
but the optimization over $\Psi$ is exact.  However, as discussed
below, continuing to add a new copy of $\Psi(x)$ on each round may be
undesirable in some cases; to address this, we analyze loss functions
of the form $f_t(x) + \alpha_t \Psi(x)$ where $\alpha_t$ is a
non-increasing sequence of non-negative numbers.  This is useful, for
example, if one wishes to encode a Bayesian prior in the online
setting (see Section~\ref{sec:motivation}).  Our proof technique has
the advantage that handling this general form of composite updates
requires only a few extra lines beyond the non-composite proof.
The original analysis of FTPRL by~\emcite{mcmahan10boundopt} did not
support composite updates.  In addition to remedying this, we prove a
new stronger version of the ``FTRL/BTL Lemma'' which tightens the
analysis of FTPRL by a constant factor.  The new lemma is quite
general and may be of independent interest.

Our unified analysis relies on a formulation of all of these
algorithms as instances of follow-the-regularized-leader,
which we develop in Section~\ref{sec:equiv}.  A preliminary version of
these equivalence results appeared in~\cite{mcmahan10equiv}.  Our
equivalence theorems apply to algorithms that use arbitrary strongly
convex regularization; however, these results show that the most
interesting strict equivalences occur in the case of quadratic
regularization.  Thus, for the analysis of Section~\ref{sec:analysis}
we restrict attention to this case, namely to algorithms where the
incremental strong convexity is of the form
\[
  R_t(x) = \hnorm{Q_t^\h (x - y)}^2_2
\] 
where $y \in R^n$ and $Q_t$ is a positive-semidefinite matrix.  This is
less general than previous results in terms of arbitrary
strongly-convex functions or Bregman divergences.  

\paragraph{Application to Sparse Models via $L_1$ Regularization }
On the surface, follow-the-regularized-leader algorithms like
regularized dual averaging~\cite{xiao09dualaveraging} appear quite
different from gradient descent (and more generally, mirror descent)
style algorithms like FOBOS~\cite{duchi09fobos}. However, the results
of Section~\ref{sec:equiv} show that in the case of quadratic
stabilizing regularization there are only two differences
between the algorithms:
\begin{itemize} 
\item How they choose to center the additional strong convexity used
  to guarantee low regret: RDA centers this regularization at the
  origin, while FOBOS centers it at the current feasible point.

\item How they handle an arbitrary non-smooth regularization function
  $\Psi$.  This includes the mechanism of projection onto a feasible
  set and how $L_1$ regularization is handled.
\end{itemize}
To make these differences precise while also illustrating
that these families are actually closely related, we consider a third
algorithm, \prox.  When the non-smooth  term $\Psi$ is
omitted, this algorithm is in fact identical to \fobos.  On the other
hand, its update is essentially the same as that of dual averaging,
except that additional strong convexity is centered at the current
feasible point (see Table~\ref{table:algorithms}).

Previous work has shown experimentally that dual averaging with $L_1$
regularization is much more effective at introducing sparsity than
FOBOS~\cite{xiao09dualaveraging,duchi10adaptive}.  Our equivalence
theorems provide a theoretical explanation for this: while RDA
considers the cumulative $L_1$ penalty $t \lambda \smnorm{x}_1$ on round
$t$, FOBOS (when viewed as a global optimization using our equivalence
theorem) considers $\phi_{1:t-1} \cdot x +\lambda \smnorm{x}_1$, where
$\phi_s$ is a certain subgradient approximation of $\lambda
\smnorm{x_s}_1$ (we use $\phi_{1:t-1}$ as shorthand for
$\sum_{s=1}^{t-1} \phi_s$, and extend the notation to sums over
matrices and functions as needed).

An experimental comparison of \fobos, \orgn, and \prox, presented in
Section~\ref{sec:exp}, demonstrates the validity of the above
explanation.  The \prox algorithm behaves very similarly to \orgn in
terms of sparsity, confirming that it is the cumulative subgradient
approximation to the $L_1$ penalty that causes decreased sparsity in
FOBOS.

In recent years, online gradient descent and stochastic gradient
descent (its batch analogue) have proven themselves to be excellent
algorithms for large-scale machine learning.  In the simplest case
\prox is identical, but when $L_1$ or other non-smooth regularization
is needed, \prox significantly outperforms FOBOS, and can outperform
\orgn as well.  Since the implementations of \prox and \orgn only
differ by a few lines of code, we recommend trying both and picking
the one with the best performance in practice.

\begin{table*}[t!]
\begin{center}
\[
\begin{array}{rcc*{3}{l@{\qquad}}}
             &           &\  (A)                    &  \qquad (B)                                  &  \qquad (C)  \\
\vspace{2pt}\text{COMID} & \argmin_x & g'_{1:t-1} \cdot x + f_t(x) & + \ \ \phi_{1:t-1} \cdot x + \alpha_t \Psi(x)  &+ \frac{1}{2}\sum_{s=1}^t\smnorm{ Q_s^{\frac 1 2}(x  - x_s)}^2  \\
\vspace{2pt}\text{RDA}   & \argmin_x & g'_{1:t-1} \cdot x + f_t(x) & + \ \ \alpha_{1:t} \Psi(x)                    &+ \frac{1}{2}\sum_{s=1}^t\smnorm{ Q_s^{\frac 1 2}(x - 0)}^2     \\
\text{FTPRL} & \argmin_x & g'_{1:t-1} \cdot x + f_t(x) & + \ \ \alpha_{1:t} \Psi(x)                    &+ \frac{1}{2}\sum_{s=1}^t\smnorm{ Q_s^{\frac 1 2}(x  - x_s)}^2  \\
\text{AOGD}  & \argmin_x & g'_{1:t-1} \cdot x + f_t(x) & + \ \ \phi_{1:t-1} \cdot x + \Psi(x)  & + \frac{1}{2}\sum_{s=1}^t\smnorm{ Q_s^{\frac 1 2}(x - 0)}^2_2    \\
\end{array}                                                                    
\]                                                                             
  \caption{ The algorithms considered in this paper, expressed as
    particular instances of the update of Eq.~\eqref{eq:update}.  The
    fact that we can express COMID and adaptive online gradient
    descent (AOGD) in this way is a consequence of
    Theorems~\ref{thm:gd-is-proxftrl} and~\ref{thm:gen-gdfprime}.
    Each algorithms' objective has three components: (A) An
    approximation to the sum of previous loss functions $f_{1:t}$,
    where the first $t-1$ functions are approximated by linear terms,
    and $f_t$ is included exactly (exactly including $f_t$ make the
    updates implicit).  (B) Terms for the non-smooth composite terms
    $\alpha_t \Psi$.  COMID approximates the terms for $\alpha_{1:t-1}
    \Psi$ by subgradients, while RDA and FTPRL consider them exactly.
    And finally, (C), stabilizing regularization needed to ensure low
    regret.  }
\label{table:algorithms}
\end{center}
\end{table*}

\section{Algorithms and Regret Bounds}\label{sec:algs}
We begin by establishing notation and introducing more formally the
algorithms we consider.  We consider loss functions $f_t(x) + \alpha_t
\Psi(x)$, where $\Psi$ is a fixed (typically non-smooth)
regularization function.  In a typical online learning setting, given
an example $(\theta_t, y_t)$ where $\theta_t \in \R^n$ is a feature
vector and $y_t \in \{-1, 1\}$ is a label, we take $f_t(x) =
\text{loss}(\theta_t \cdot x, y_t)$.  For example, for logistic
regression we use log-loss, $\text{loss}(\theta_t \cdot x, y_t) =
\log(1 + \exp(-y_t \theta_t \cdot x)).$ 
All of the algorithms we
consider support composite updates (consideration of $\Psi$ explicitly
rather than through a gradient $\grad f_t(x_t)$) as well as positive
semi-definite matrix learning rates $Q$ which can be chosen adaptively
(the interpretation of these matrices as learning rates will be
clarified in Section~\ref{sec:equiv}).

We first consider the specific algorithms used in the $L_1$
experiments of Section~\ref{sec:exp}; we use the standard reduction to
linear functions, letting $g_t = \grad f_t(x_t)$.  The first algorithm
we consider is from the gradient-descent family, namely \fobos, which
plays
\begin{equation*}
\xti = \argmin_x g_t \cdot x + \lambda \norm{x}_1 + \frac{1}{2}
\norm{ Q_{1:t}^{\frac 1 2}(x - x_t)}^2_2.  
\end{equation*}
We state this algorithm implicitly as an optimization, but a
gradient-descent style closed-form update can also be
given~\citep{duchi09fobos}.  The algorithm was described in this
form as a specific composite-objective mirror descent (COMID)
algorithm by~\citet{duchi10composite}.

The regularized dual averaging (\orgn) algorithm of \citet{xiao09dualaveraging} plays
\begin{equation*}
\xti = \argmin_x  g_{1:t} \cdot x + t \lambda \norm{x}_1 
         + \frac{1}{2}\sum_{s=1}^t\norm{ Q_s^{\frac 1 2}(x - 0)}^2_2.
\end{equation*}
In contrast to \fobos, the RDA optimization is over the sum $g_{1:t}$
rather than just the most recent gradient $g_t$.  We will show (in
Theorem~\ref{thm:gen-gdfprime}) that when $\lambda = 0$ and the $f_t$
are not strongly convex, this algorithm is in fact equivalent to the
adaptive online gradient descent (AOGD) algorithm of
~\emcite{bartlett07adaptive}.

RDA is directly defined as a FTRL algorithm, and hence is also an
instance of the more general primal-dual algorithmic schema of
\citet{shwartz06repeated}; see also \citet{kakade09smoothness}.
However, these general results are not sufficient to prove the
original bounds for RDA, nor the versions here that extend to implicit
updates.

The \prox algorithm plays
\begin{equation*}
\xti = \argmin_x  g_{1:t} \cdot x + t \lambda \norm{x}_1 
   + \frac{1}{2}\sum_{s=1}^t\norm{ Q_s^{\frac 1 2}(x  - x_s)}^2_2.
\end{equation*}
This algorithm was introduced by~\emcite{mcmahan10boundopt}, but without
support for an explicit $\Psi$.

One of our principle contributions is showing the close connection
between all four of these algorithms; Table~\ref{table:algorithms}
summarizes the key results from Theorems~\ref{thm:gd-is-proxftrl} and
\ref{thm:gen-gdfprime}, writing AOGD and FOBOS in a form that makes
the relationship to \orgn and \prox explicit.

In our equivalence analysis, we will consider arbitrary convex
functions $R_t$ and $\Rp_t$ in place of the $\h \norm{Q_t^{\frac 1
    2}x}_2^2$ and $\h \norm{ Q_t^{\frac 1 2}(x - x_t)}^2_2$ that
appear here, as well as arbitrary convex $\Psi(x)$ in place of
$\lambda \smnorm{x}_1$.

\subsection{Implicit and Composite Updates for FTRL}
The algorithms we consider can be expressed as
follow-the-regularized-leader (FTRL) algorithms that perform implicit
and composite updates.  The standard subgradient FTRL algorithm uses
the update
\[
\xti = \argmin_x \left( \sum_{s=1}^{t} \grad f_s(x_s)\right) \cdot x 
  + R_{1:t}(x).
\]
In this update, each previous (potentially non-linear) loss function
$f_s$ is approximated by the gradient at $x_s$ (when $f_s$ is not
differentiable, we can use a subgradient at $x_s$ in place of the
gradient).  The functions $R_t$ are incremental regularization added
on each round; for example $R_{1:t}(x) = \sqrt{t}\smnorm{x}^2$ is a
standard choice, corresponding to regularized dual averaging.

Implicit update rules are usually defined for mirror descent
algorithms, but we can define an analogous update for FTRL:
\begin{equation*}
\xti = \argmin_x \left( \sum_{s=1}^{t-1} \grad f_s(x_{s+1})\right) \cdot x 
+ f_t(x) + R_{1:t}(x).
\end{equation*}
This update replaces the subgradient approximation of $f_t$ with the
possibly non-linear $f_t$.  Closed-form implicit updates for the
squared error case were derived by~\cite{kivinen94exponentiated}; the
term implicit updates was coined later~\cite{kivinen06pnorm}.  Our
formulation is similar to the online coordinate-dual-ascent algorithm
briefly mentioned by~\emcite{shwartz08mind}.  In general, computing the
implicit update might require solving an arbitrary convex optimization
problem (hence, the name implicit), however, in many useful
applications it can be computed in closed form or by optimizing a
one-dimensional problem.  We discuss the advantages of implicit
updates in Section~\ref{sec:motivation}.

Analysis of implicit updates has proved difficult.
\emcite{kulis10implicit} provide the only other regret bounds for
implicit updates that match those of the explicit-update versions.
While their analysis handles more general divergences, it only applies
to mirror-descent algorithms.  Our analysis handles composite
objectives and applies FTRL algorithms as well as mirror descent.  Our
analysis also quantifies the one-step improvement in the regret bound
obtained by the implicit update, showing the inequality is in fact
strict when the implicit update is non-trivial.

When $f_t$ is not differentiable, we use the update
\begin{equation}\label{eq:implicit_only}
\xti = \argmin_x g'_{1:t-1} \cdot x + f_t(x) + R_{1:t}(x),
\end{equation}
where $g'_t$ is a subgradient of $f_t$ at $\xti$ (that is, $g'_t \in
\partial f_t(\xti)$) such that
$ g'_{1:t-1} + g_t' + \grad \Rtt(\xti) = 0.$
The existence of such a subgradient is proved below, in Theorem~\ref{thm:std}.

In many applications, we have a fixed convex function $\Psi$ that we
also wish to include in the optimization, for example $\Psi(x) =
\smnorm{x}_1$ ($L_1$-regularization to induce sparsity) or the indicator
function on a feasible set $\Fs$ (see Section~\ref{sec:notation}).
While it is possible to approximate this function via subgradients as
well, when computationally feasible it is often better to handle
$\Psi$ directly.  For example, in the case where $\Psi(x) =
\smnorm{x}_1$, subgradient approximations will in general not lead to
sparse solutions.  In this case, closed-form updates for optimizations
including $\Psi$ are often possible, and produce much better
sparsity~\cite{xiao09dualaveraging,duchi09fobos}.  We can include such
a term directly in FTRL, giving the composite objective update
\begin{equation}\label{eq:composite_only}
\xti = \argmin_x \left( \sum_{s=1}^t \grad f_s(x_s)\right) \cdot x 
 + \alpha_{1:t} \Psi(x) + R_{1:t}(x),
\end{equation}
where $\alpha_t$ is the weight on $\Psi$ on round $t$.

This formulation, which allows for an arbitrary sequence of
non-negative, non-increasing $\alpha_t$'s, is more general than that
supported by the original analysis of COMID or RDA.  \citet[Sec
6.1]{xiao10dualavgjmlr} shows that RDA does allow a varying schedule
where $\alpha_{1:t} = c + 1/\sqrt{t}$ for a constant $c$, by
incorporating part of the $\Psi$ term in the regularization function
$R_t$; this is less general than our analysis, which allows the
schedule $\alpha_t$ to be chosen independently of the learning rate.

Finally, we can combine these ideas to define an implicit update with
a composite objective.  In the general case where $f_t$ is not
differentiable, we have the update
\begin{equation}\label{eq:update}
\xti = \argmin_x g'_{1:t-1} \cdot x + f_t(x) + \alpha_{1:t} \Psi(x) + R_{1:t}(x),
\end{equation}
where $g'_t \in \partial f_t(\xti)$ such that $\exists \phi_t \in \partial \Psi(\xti)$
where
$ g'_{1:t-1} + g'_t + \phi_t + \grad \Rtt(\xti) = 0.$
The existence of such a subgradient again follows from
Theorem~\ref{thm:std}.

\newcommand{\fw}{f^w} 
\newcommand{\fu}{f^u} 

It is worth noting that our analysis of implicit updates applies
immediately to standard first-order updates.  Let $\fw_t$ designate
the loss function provided by the \emph{world}, and let $\fu_t$ be the
loss function in the \emph{update} Eq.~\eqref{eq:update}.  Then we
recover the non-implicit algorithms by taking $\fu_t(x) \leftarrow
\grad \fw_t(x_t) \cdot x$.

\subsection{
Motivation for Implicit Updates and  Composite Objectives}
\label{sec:motivation}
Implicit updates offer a number of advantages over using a subgradient
approximation.  \emcite{kulis10implicit} discusses several important
examples.  They also observe that empirically, implicit updates
outperform or nearly outperform linearized updates, and show more
robustness to scaling of the data.

Learning problems that use importance weights on examples are also a
good candidate for implicit updates.  Importance weights can
be used to compress the training data, by replacing $n$ copies of an
example with one copy with weight $n$.  They also arise in active
learning algorithms~\cite{beygelzimer10agnostic} and situations where
the training and test distributions differ (covariate
shift, e.g.~\emcite{sugiyama08direct}).  Recent work has demonstrated
experimentally that implicit updates can significantly outperform
first-order updates both on importance weighted and standard learning
problems~\cite{karampatziakis10importance}.

The following simple examples demonstrates the intuition for these
improvements.  The key is that the linearization of $f_t$
over-estimates the decrease in loss under $f_t$ achieved by moving in
the direction $\grad f_t(x_t)$.  The farther $x_{t+1}$ is chosen from
$x_t$, and the more non-linear the $f_t$, the worse this approximation
can be.  Consider gradient descent in one dimension with $f_t(x) =
\h(x - 3)^2$ and $x_t = 2$.  Then $\grad f_t(2) = -1$, and if we
choose a learning rate $\eta_t > 1$, we will actually overshoot the
optimum for $f_t$ (such a learning rate could be indicated by the
theory if the feasible set is large, for example).  Implicit updates,
on the other hand, will never choose $\xti > 3$, rather $\xti
\rightarrow 3$ as $\eta_t \rightarrow \infty$.  Thus, we see implicit
updates can be significantly better behaved with large learning rates.
Note that an importance weight of $n$ is equivalent to multiplying the
learning rate by $n$, so when importance weights can be large, implicit
updates can be particularly beneficial.

The overshooting issue is even more pronounced with non-smooth
objectives, for example, $f_t(x) = g\cdot x + \smnorm{x}_1$.  A standard
gradient descent update will in general never set $x_{t+1} = 0$ despite
the $L_1$ regularization; handling the $L_1$ term via an implicit
update solves this problem.  This is exactly the insight that COMID
algorithms like FOBOS exploit; by analyzing general implicit updates,
we achieve an analysis of these algorithms while also supporting a
much larger class of updates.

When the functional form of the non-smooth component of the objective
(for example $\smnorm{x}_1$) is fixed across rounds, it is preferable to
perform an explicit optimization involving the total non-smooth
contribution $\alpha_{1:t}\Psi$ (RDA and FTPRL) rather than just the
round $t$ contribution $\alpha_t \Psi$ (COMID).  While RDA supports
this type of non-smooth objective, it requires the weight on $\Psi$ to
be fixed across rounds.  We generalize this to non-increasing
per-round contributions in this work.

Suppose one is performing online logistic regression, and believes a
priori that the coefficients have a Laplacian distribution.  Then,
$L_1$-penalized logistic regression corresponds to MAP
estimation~(e.g., \emcite{lee06efficient}); suppose the prior
corresponds to a total penalty of $\lambda \smnorm{x}_1$.  If the size
of the dataset $T$ is known in advance, then we can use $\alpha_t =
\lambda/T$, and by making multiple passes over the data, we will
converge to the MAP estimate.  However, in the online setting we will
in general not know $T$ in advance, and we may wish to use an online
algorithm for computational efficiency.  In this case, any fixed value
of $\alpha_t$ will correspond to strengthening the prior each time we
see a new example, which is undesirable.  With the generalized notion
of composite updates introduced here, this problem is overcome by
choosing $\alpha_1 \Psi(x) = \lambda \smnorm{x}_1$, and $\alpha_t = 0$ for $t
\geq 2$.  Thus, the fixed penalty on the coefficients is correctly
encoded, independent of $T$.

\subsection{Summary of Regret Bounds}
In Section~\ref{sec:analysis}, we analyze the update rule of
Equation~\eqref{eq:update} when
\begin{equation}\label{eq:quadreg}
R_t(x) = \frac{1}{2}\norm{ Q_t^{\frac 1 2}(x - y_t)}^2,
\end{equation}
where $\smnorm{\cdot} = \smnorm{\cdot}_2$ here and throughout.  The points
$y_t \in \R^n$ are the centers for the additional regularization added
on each round.  Choosing $y_t = 0$ leads to an analysis of RDA with
implicit updates, and choosing $y_t = x_t$ yields the
follow-the-proximally-regularized-leader algorithm with implicit
updates.  Using $y_t = x_t$ together with a modified choice of $f_t$
leads to composite-objective mirror descent (see
Section~\ref{sec:equiv}).

The generalized learning rates $Q_t$ can be chosen adaptively using
techniques from \emcite{mcmahan10boundopt} and
\emcite{duchi10adaptive}, which leads to improved regret bounds, as
well as algorithms that perform much better in
practice~\cite{streeter10conditioning}.  Since in this work we provide
suitable regret bounds in terms of arbitrary $Q_t$, the adaptive
techniques can be applied directly.  Doing so complicates the
exposition somewhat, and so for simplicity and easy of comparison to
previous results we state specific regret bounds for scalar learning
rates:
\begin{corollary}\label{cor:main}
  Let $\Psi$ be the indicator function on a feasible set $\Fs$, and
  let $D = \max_{a,b \in \Fs}\smnorm{a - b}$.  So that our bounds are
  comparable, suppose $\max_{a \in \Fs} \smnorm{a} = \frac{D}{2}$ (for
  example, if $\Fs$ is symmetric).  Let $f_t$ be a sequence of convex
  loss functions such that $\smnorm{\grad f_t(x)} \leq G$ for all $t$ and all
  $x \in \Fs$.  Then for FTPRL we set $x_t = y_t$ and have
  \[\Regret \leq DG\sqrt{2T}.\]
  Implicit-update mirror descent obtains the same bound.
  For regularized dual averaging we choose $y_t = 0$ for all $t$, and obtain
  \[\Regret \leq \frac{1}{2}DG\sqrt{2T} +  \frac{G D}{\sqrt{2}} \ln T + \BO(1).\]
\end{corollary}
These bounds are achieved with an adaptive learning rate that depends
only on $t$ ($T$ need not be known in advance).  If $T$ is known, then
the $\sqrt{2}$ constant on the $\sqrt{T}$ terms can be eliminated.
The regret bounds with per-coordinate adaptive rates are at least as
good, and often better.  This corollary is a direct consequence of the
following general result:

\newcommand{\xn}{\hat{x}}

\newcommand{\gnt}{\hat{g}_t}
\newcommand{\flnR}{\fln^R}

\begin{theorem}\label{thm:genmain}
  Let $\Psi$ be an extended convex function on $\R^n$ with $\Psi(x)
  \geq 0$ and $0 \in \partial \Psi(0)$, let $f_t$ be a sequence of
  convex loss functions, and let $\alpha_t \in \R$ be non-negative
  and non-increasing real numbers ($0 \leq \alpha_{t+1} \leq
  \alpha_t$).  Consider the FTRL algorithm that plays $x_1 = 0$ and
  afterwards plays according to Equation~\eqref{eq:update},
  \begin{equation*}
    \xti = \argmin_x g'_{1:t-1} \cdot x + f_t(x) 
    + \alpha_{1:t} \Psi(x) + R_{1:t}(x),
  \end{equation*}
using
  incremental quadratic regularization functions $R_t(x) =
  \hnorm{Q_t^\h(x - y_t)}^2$ where $Q_1 \in \Snpp$, $Q_t \in \Snp$ for
  $t > 1$, and $y_t \in \R^n$.  Then there exist $\wg_t \in R^n$ such that 
  \begin{align*} 
  \Regret(f)   
   &\leq R_{1:T}(\xs) + \alpha_{1:T}\Psi(\xs) 
         + \sum_{t=1}^T (g_t - \h \wg_t)^\tp \Qtt^{-1} \wg_t  - g_t \Qtt\inv Q_t (y_t - x_t) \\
   &\leq R_{1:T}(\xs) + \alpha_{1:T}\Psi(\xs) 
         + \sum_{t=1}^T \hnorm{\Qtt\invh g_t}^2  - \delta_{1:t}
         - g_t \Qtt\inv Q_t (y_t - x_t) 
  \end{align*}
  versus any point $\xs \in \R^n$, for any $g_t \in \partial
  f_t(x_t)$, with $\delta \geq 0$.
\end{theorem}
We will show that $\wg_t$ is a certain subgradient of $f_t$, and in
fact when all $\alpha_t = 0$, then $\wg_t \in \partial f_t(x_{t+1})$.
If $f_t$ is strictly convex, then in general $\wg_t \neq g_t$, and so
the inequality between the first and second bounds can be strict; in
fact, we will show that on rounds $t$ where the implicit-update is
non-trivial, $\delta > 0$, indicating a one-step advantage for
implicit updates.  When all $\alpha_t = 0$, $\delta_t$ is one-half the
improvement in the objective function of Equation~\eqref{eq:update}
obtained by solving for the optimum point rather than using a solution
from the linearized problem; the proof of Lemma~\ref{lem:newQ} makes
this precise.

For RDA, we take $y_t = 0$, and for FTPRL and implicit-update mirror
descent we take $y_t = x_t$.  Since no restrictions are placed on the
$y_t$ in the theorem, the final right-hand term being subtracted could
have be positive, negative, or zero.

If we treat $\alpha_t \Psi$ as an intrinsic part of the problem, that
is, we are measuring loss against $f_t(x) + \alpha_t \Psi(x)$, then
the $\alpha_{1:T} \Psi(\xs)$ term disappears from the regret bound.

\subsection{Notation and Technical Background}
\label{sec:notation}
We use the notation $\gtt$ as a shorthand for $\sum_{s=1}^t
g_s$. Similarly we write $\Qtt$ for a sum of matrices $Q_t$, and we
use $f_{1:t}$ to denote the function $f_{1:t}(x) = \sum_{s = 1}^t
f_s(x)$.  We assume the summation binds more tightly than exponents, so
$Q_{1:t}^\h = (Q_{1:t})^\h$.  We write $x^\tp y$ or $x \cdot y$ for the
inner product between $x,y \in \R^n$.  We write ``the functions
$f_t$'' for the sequence of functions $(f_1, \dots, f_T)$.

We write $\Snp$ for the set of symmetric positive semidefinite $n
\times n$ matrices, with $\Snpp$ the corresponding set of symmetric
positive definite matrices. Recall $A \in \Snpp$ means $\forall x\neq
0,\ x^\tp Ax > 0$.  Since $A \in \Snp$ is symmetric, $x^\tp A y =
y^\tp A x$ (we often use this result implicitly).
For $B \in
\Snp$, we write $B^{1/2}$ for the square root of $B$, the unique $X \in
\Snp$ such that  $XX = B$ (see, for example,~\emcite[A.5.2]{boyd}).

Unless otherwise stated, convex functions are assumed to be extended,
with domain $\R^n$ and range $\R \cup \set{\infty}$ (see, for
example~\cite[3.1.2]{boyd}).  For a convex function $f$, we let $\partial
f(x)$ denote the set of subgradients of $f$ at $x$ (the
subdifferential of $f$ at $x$).  By definition, $g \in
\partial f(x)$ means $ f(y) \geq f(x) + g^\tp(y - x) $ for all $y$.
When $f$ is differentiable, we write $\grad f(x)$ for the gradient of
$f$ at $x$.  In this case, $\partial f(x) = \set{\grad f(x)}$.  All
mins and argmins are over $\R^n$ unless otherwise noted.
We make frequent use of the following standard results, summarized as
follows:
\begin{theorem}\label{thm:std}
  Let $R:\R^n \rightarrow \R$ be strongly convex with continuous first
  partial derivatives, and let $\Phi$ and $f$ be arbitrary
  (extended) convex functions.  Then,
  \begin{enumerate}[A.]
  \item Let  $U(x) = R(x) + \Phi(x)$. Then,
    there exists a unique pair $(x^*, \phi^*)$ such that both
    \[ 
    \qquad \phi^* \in \partial \Phi(x^*) 
    \qquad \text{and} \qquad 
    \qquad x^* = \argmin_x R(x) + \phi^* \cdot x.
    \] 
    Further, this $x^*$ is the unique minimizer of $U$, and $\grad R(x^*) +
    \phi^* = 0$.
  \item Let $V(x) = R(x) + \Phi(x) + f(x)$ and $\xs = \argmin_x V(x)$.
    Then, there exists a $g \in \partial f(\xs)$ such that
    \[ \xs = \argmin_x R(x) + \Phi(x) + g\cdot x.\]
  \end{enumerate}
\end{theorem}

\begin{proof}
  First we consider part $A$.  Since $R$ is strongly convex, $U$ is
  strongly convex, and so has a unique minimizer $x^*$ (see for
  example, \cite[9.1.2]{boyd}).  Let $r = \grad R$.  Since $x^*$ is a
  minimizer of $U$, there must exist a $\phi^* \in \partial \Phi(x^*)$
  such that $r(x^*) + \phi^* = 0$, as this is a necessary (and
  sufficient) condition for $0 \in \partial U(x^*)$.  It follows that
  $x^* = \argmin_x R(x) + \phi^* \cdot x$, as $r(x^*) + \phi^*$ is the
  gradient of this objective at $x^*$.  Suppose some other $(x',
  \phi')$ satisfies the conditions of the theorem.  Then, $r(x') +
  \phi' = 0$, and so $0 \in \partial U(x')$, and so $x'$ is a
  minimizer of $U$.  Since this minimizer is unique, $x' = x^*$, and
  $\phi' = -r(x^*) = \phi^*$.  An equivalent condition to $x^* =
  \argmin_x R(x) + \phi^* \cdot x$ is $\grad R(x^*) + \phi^* = 0$.
  
  \newcommand{\hphi}{\hat{\phi}} For part $B$, by definition of
  optimality, there exists a $\phi \in \partial \Phi(\xs)$ and a $g
  \in \partial f(\xs)$ such that $g + \phi + \grad R(\xs) = 0$.
  Choosing this $g$, define
  \[ \hx = \argmin_x R(x) + \Phi(x) + g \cdot x.
  \] 
  Applying part $A$ with $R(x) \leftarrow R(x) + g\cdot x$, there
  exists a unique pair $(\hx, \hphi)$ such that $\hphi \in \partial
  \Phi(\hx)$ and $\grad R(\hx) + \hphi + g = 0$.  Since $(\xs, \phi)$
  satisfy this equation, we conclude $\xs = \hx$.
\end{proof}

\paragraph{Feasible Sets}
In some applications, we may be restricted to only play points from a
convex feasible set $\Fs \subseteq \R^n$, for example, the set of
(fractional) paths between two nodes in a graph.  A feasible set is
also necessary to prove regret bounds against linear functions.
With composite updates, Equations~\eqref{eq:composite_only} and
\eqref{eq:update}, this is accomplished for free by choosing $\Psi$ to
be the indicator function $\IF$ on $\Fs$, where $\IF(x) = 0$ for $x
\in \Fs$ and $\infty$ otherwise.
It is straightforward to verify that 
\[ 
  \argmin_{x \in \R^n} g_{1:t} \cdot x + R_{1:t}(x) + \IF(x)
   \quad = \quad
  \argmin_{x \in \Fs} g_{1:t} \cdot x + R_{1:t}(x),
\] 
and so in this work we can generalize (for example) the results
of~\cite{mcmahan10boundopt} for specific feasible sets without
specifically discussing $\Fs$, and instead considering arbitrary
extended convex functions $\Psi$.  Note that in this case the choice
of $\alpha_t$ does not matter as long as $\alpha_1 > 0$.

\section{Mirror Descent Follows The Leader}
\label{sec:equiv}

In this section we consider the relationship between mirror descent
algorithms (the simplest example being online gradient descent) and
FTRL algorithms.  Let $f_t(x) = g_t \cdot x + \Psi(x)$.

Let $R_1$ be strongly convex, with all the
$R_t$ convex.  We assume that $\min_x R_1(x) = 0$, and
assume that $x=0$ is the unique minimizer unless otherwise noted.

\paragraph{Follow The Regularized Leader (FTRL)} 
The simplest follow-the-regularized-leader algorithm plays
\begin{equation}\label{eq:simple-ftrl}
 x_{t+1} = \argmin_{x} g_{1:t} \cdot x + \frac{\sigma_{1:t}}{2}\smnorm{x}^2_2,
\end{equation}
where $\sigma_{1:t} \in \R$ is the amount of stabilizing strong convexity added.

A more general update is
\begin{equation*} %
 x_{t+1} = \argmin_{x} g_{1:t} \cdot x + R_{1:t}(x).
\end{equation*}
where we add an additional convex function $R_t$ on each round.  When
$\argmin_{x \in \R^n} R_t(x) = 0,$ we call the functions $R_t$ (and
associated algorithms) \emph{origin-centered}.  We can also define
\emph{proximal} versions of FTRL\footnote{We adapt the name
  ``proximal'' from~\cite{do09proximal}, but note that while similar
  proximal regularization functions were considered, that paper deals
  only with gradient descent algorithms, not FTRL.} that center
additional regularization at the current point rather than at the
origin.  In this section, we write $\Rp_t(x) = R_t(x - x_t)$ and
reserve the $R_t$ notation for origin-centered functions.  Note that
$\Rp_t$ is only needed to select $x_{t+1}$, and $x_t$ is known to the
algorithm at this point, ensuring the algorithm only needs access to
the first $t$ loss functions when computing $\xti$ (as required).

\paragraph{Mirror Descent}

The simplest version of mirror descent is gradient descent using a
constant step size $\eta$, which plays
\begin{equation}\label{eq:simple-gd}
 x_{t+1} = x_t  - \eta g_t = -\eta g_{1:t}.
\end{equation}

In order to get low regret, $T$ must be known in advance so $\eta$ can
be chosen accordingly (or a doubling trick can be used).  But, since
there is a closed-form solution for the point $x_{t+1}$ in terms of
$g_{1:t}$ and $\eta$, we generalize this to a ``revisionist''
algorithm that on each round plays the point that gradient descent
with constant step size would have played if it had used step size
$\eta_t$ on rounds $1$ through $t-1$.  That is, $ x_{t+1} = -\eta_t
g_{1:t}.$   When $R_t(x) = \frac{\sigma_t}{2}\smnorm{x}^2_2$ and $\eta_t =
\frac{1}{\sigma_{1:t}}$, this is equivalent to the FTRL of
Equation~\eqref{eq:simple-ftrl}.

In general, we will be more interested in gradient descent algorithms
which use an adaptive step size that depends (at least) on the
round $t$.  Using a variable step size $\eta_t$ on each round,
gradient descent plays:
\begin{equation}\label{eq:simple-gd-adapt}
 x_{t+1} = x_t - \eta_t g_t.
\end{equation}
An intuition for this update comes from the fact it can be re-written as
\[
x_{t+1} = \argmin_x g_t \cdot x 
+ \frac{1}{2\eta_t}\smnorm{x - x_t}_2^2.
\] 
This version captures the notion (in online learning
terms) that we don't want to change our hypothesis $x_t$ too much (for
fear of predicting badly on examples we have already seen), but we do
want to move in a direction that decreases the loss of our hypothesis
on the most recently seen example.  Here, this is approximated by the
linear function $g_t$, but implicit updates use the exact loss $f_t$.

Mirror descent algorithms use this intuition, replacing the
$L_2$-squared penalty with an arbitrary Bregman divergence.  For a
differentiable, strictly convex  $R$, the corresponding Bregman divergence is 
\[ \Br_R(x, y) = R(x) - \big(R(y) + \grad R(y)\cdot (x - y)\big)\]
for any $x,y \in \R^n$.
We then have the update
\begin{equation} \label{eq:gen-gd-exact}
 x_{t+1} = \argmin_{x} g_t \cdot x + \frac{1}{\eta_t}\Br_R(x, x_t),
\end{equation}
or explicitly (by setting the gradient of~\eqref{eq:gen-gd-exact} to zero),
\begin{align} 
x_{t+1} = r^{-1}( r(x_t) - \eta_t g_t) \label{eq:gen-gd}
\end{align}
where $r = \grad R$.  Letting $R(x) = \frac{1}{2}\smnorm{x}^2_2$ so that
$\Br_R(x, x_t) = \h \smnorm{x - x_t}_2^2$ recovers the algorithm of
Equation~\eqref{eq:simple-gd-adapt}.  One way to see this is to note
that $r(x) = r^{-1}(x) = x$ in this case.

We can generalize this even further by adding a new strongly convex
function $R_t$ to the Bregman divergence on each round.  Namely, let
\[
\Brtt(x, y) = \sum_{s=1}^t \Br_{R_s}(x, y),\]
 so the update becomes
\begin{equation}\label{eq:md}
 \xti = \argmin_x g_t \cdot x + \Brtt(x, x_t)
\end{equation}
or equivalently 
$
x_{t+1} = (r_{1:t})^{-1}( r_{1:t}(x_t) - g_t)
$
where $r_{1:t} = \sum_{s=1}^t \grad R_t = \grad R_{1:t}$ and 
$(r_{1:t})^{-1}$ is the inverse of $r_{1:t}$. 
The step size $\eta_t$ is now encoded implicitly in the choice of
$R_t$.

Composite-objective mirror descent (COMID)~\cite{duchi10composite} handles
$\Psi$ functions\footnote{Our $\Psi$ is denoted $r$
  in~\cite{duchi10composite}} as part of the objective on each round:
$f_t(x) = g_t \cdot x + \Psi(x)$.  Using our notation, the COMID
update is
\[
\xti = \argmin_x \eta g_t \cdot x + \Br(x, x_t) + \eta \Psi(x),
\]
which can be generalized to 
\begin{equation}\label{eq:gen-comid}
 \xti = \argmin_x g_t \cdot x + \Psi(x) + \Brtt(x, x_t),
\end{equation}
where the learning rate $\eta$ has been rolled into the definition of
$R_1, \dots, R_t$. 
When $\Psi$ is chosen to be the indicator function on a convex set,
COMID reduces to standard mirror descent with greedy projection.

\subsection{An Equivalence Theorem for Proximal Regularization}

The following theorem shows that mirror descent algorithms can be
viewed as FTRL algorithms:

\newcommand{\Brp}{\tilde{\Br}}  %
\newcommand{\Brptt}{\Brp_{1:t}}
\newcommand{\Brpt}{\Brp_t}

\begin{theorem} \label{thm:gd-is-proxftrl}
  Let $R_t$ be a sequence of differentiable origin-centered convex
  functions $(\grad R_t(0) = 0)$, with $R_1$ strongly convex, and let
  $\Psi$ be an arbitrary convex function.  Let $x_1 = \hx_1 = 0$.  For
  a sequence of loss functions $f_t(x) + \Psi(x)$, let the sequence of
  points played by the implicit-update composite-objective mirror
  descent algorithm be
  \begin{equation}\label{eq:comidF}%
    \hx_{t+1} = \argmin_{x}\ f_t(x) + \alpha_t \Psi(x) + \Brptt(x, \hx_t),
  \end{equation}
  where $\Rp_t(x) = R_t(x - \hx_t)$, and $\Brpt = \Br_{\Rp_t}$, so
  $\Brptt$ is the Bregman divergence with respect to $\Rp_1 + \dots +
  \Rp_t$.  Consider the alternative sequence of points $x_t$ played by a
  proximal FTRL algorithm, applied to these same $f_t$, defined by
  \begin{equation}\label{eq:ftprlcomid}%
    x_{t+1} =  \argmin_{x} \ (g_{1:t-1}' + \phi_{1:t-1}) \cdot x 
    + f_t(x) + \alpha_t \Psi(x)
    + \Rp_{1:t}(x) 
  \end{equation}
  for some $g'_t \in \partial f_t(\xti)$ and $\phi_t \in \partial
  (\alpha_t \Psi)(\xti)$.  Then, these algorithms are equivalent, in
  that $x_t = \hx_t$ for all $t > 0$.
\end{theorem}
We defer the proof to the end of this section.  The Bregman
divergences used by mirror descent in the theorem are with respect to
the proximal functions $\Rp_{1:t}$, whereas typically (as in
Equation~\eqref{eq:md}) these functions would not depend on the
previous points played.  We will show when $R_t(x) = \h
\smnorm{Q_t^{\frac{1}{2}}x}^2_2$, this issue disappears.
Considering arbitrary $\Psi$ functions and implicit updates also
complicates the theorem statement somewhat.  The following corollary
sidesteps these complexities, to state a simple direct equivalence
result:

\begin{corollary}\label{cor:proximal}
Let $f_t(x) = g_t \cdot x$.  Then, the following algorithms play
identical points:
\begin{itemize} \itemsep -5pt
\item Gradient descent with positive semi-definite learning rates
  $Q_t$, defined by:
   \[x_{t+1} = x_t - Q_{1:t}^{-1} g_t.\]

   \item \prox with regularization functions $\Rp_t(x) =
   \frac{1}{2}\norm{Q_t^{\h}(x - x_t)}^2_2$, which plays
   \[ x_{t+1} = \argmin_{x} g_{1:t} \cdot x + \Rp_{1:t}(x).\]
\end{itemize}
\end{corollary}

\begin{proof}
Let $R_t(x) = \h x^\tp Q_t x$.  It is easy to show that $R_{1:t}$ and
$\Rp_{1:t}$ differ by only a linear function, and so (by a standard
result) $\Brtt$ and $\Brptt$ are equal, and simple algebra reveals
\[\Brtt(x, y) = \Brptt(x, y) = \h \smnorm{Q_{1:t}^\h (x - y)}_2^2.\]
Then, it follows from Equation~\eqref{eq:gen-gd} that the first
algorithm is a mirror descent algorithm using this Bregman divergence.
Taking $\Psi(x) = 0$ and hence $\phi_t = 0$, the result follows from
Theorem~\ref{thm:gd-is-proxftrl}.
\end{proof}

Extending the approach of the corollary to \fobos, we see the only
difference between that algorithm and \prox is that \prox optimizes
over $t \Psi(x)$, whereas in Equation~\eqref{eq:ftprlcomid} we
optimize over $\phi_{1:t-1} \cdot x + \Psi(x)$ (see
Table~\ref{table:algorithms}).  Thus, \fobos is equivalent to \prox,
except that \fobos approximates all but the most recent $\Psi$
function by a subgradient.

The behavior of \prox can thus be different from COMID when a
non-trivial $\Psi$ is used.  While we are most concerned with the
choice $\Psi(x) = \lambda \smnorm{x}_1$, it is also worth considering
what happens when $\Psi$ is the indicator function on a feasible set
$\Fs$.  Then, Theorem~\ref{thm:gd-is-proxftrl} shows that mirror
descent on $f_t(x) = g_t \cdot x + \Psi(x)$ (equivalent to COMID in
this case) approximates previously seen $\Psi$s by their subgradients,
whereas \prox optimizes over $\Psi$ explicitly.  In this
case, it can be shown that the mirror-descent update corresponds to the standard greedy
projection~\cite{zinkevich03giga}, whereas \prox corresponds to a lazy
projection \cite{mcmahan10boundopt}.\footnote{
\citet[Sec. 5.2.3]{zinkevich04thesis} describes a different lazy
projection algorithm, which requires an appropriately chosen constant
step-size to get low regret.  \prox does not suffer from this problem,
because it always centers the additional regularization $R_t$ at
points in $\Fs$, whereas our results show the algorithm of Zinkevich
centers the additional regularization \emph{outside} of $\Fs$, at the
optimum of the unconstrained optimization.  This leads to the high
regret in the case of standard adaptive step sizes, because the
algorithm can get ``stuck'' too far outside the feasible set to make
it back to the other side.}

For the analysis in Section~\ref{sec:analysis}, we will use this
special case for quadratic regularization:
\begin{corollary}\label{cor:equiv-quad}
  Consider Implicit-Update Composite-Objective Mirror Descent, which plays
  \begin{equation}%
    \hx\ti = \argmin f_t(x) + \alpha_t \Psi(x) + \hnorm{Q_{1:t}^\h (x - \hx_t)}^2.
  \end{equation}
  Then an equivalent FTPRL update is
  \begin{equation}%
    \xti = \argmin_x\ \  (g'_{1:t-1} + \phi_{1:t-1}) \cdot x + f_t(x)  + 
    \alpha_t \Psi(x) + \frac{1}{2}\sum_{s=1}^t\norm{ Q_s^{\frac 1 2}(x  - x_s)}^2
  \end{equation}
  for some $g'_t \in \partial f_t(\xti)$ and $\phi_t \in \partial
  (\alpha_t \Psi)(\xti)$.
\end{corollary}

Again let $\fw_t$ be the loss functions provided by the world, and let
$\fu_t$ be the functions defining the update and used in the above
corollary.  Then, we encode implicit mirror descent by taking $\fu_t(x)
\leftarrow \fw_t(x) + \alpha_t \Psi(x)$.  We recover standard
(non-implicit) COMID by taking $\fu_t(x) \leftarrow \grad
\fw_t(\xt)\cdot x + \alpha_t \Psi(x)$.  Applying this result leads to
the expression for COMID in Table~\ref{table:algorithms}.

Note that in both cases, the $\Psi$ listed separately in
Eq.~\eqref{eq:update} is taken to be zero; the $\Psi$ specified in the
problem only enters into the update through the $\fu_t$.  That is, we
don't actually need the machinery developed in this work for composite
updates, rather we get an analysis of mirror-descent style composite
updates via our analysis of \emph{implicit} updates.  The machinery
for explicitly handling the full $\alpha_{1:t} \Psi$ penalty should be
used in practice, however (see Section~\ref{sec:motivation}).  Note
also that the standard COMID algorithm can thus be viewed as a
half-implicit algorithm: it uses an implicit update with respect to
the $ \Psi$ term, but applies an immediate subgradient approximation
to $\fw_t$.

We conclude the section with the proof of the main equivalence result.
\begin{myproof}{Proof of Theorem~\ref{thm:gd-is-proxftrl}}
  \newcommand{\hphi}{\hat{\phi}} \newcommand{\hgp}{\hat{g}'} For
  simplicity we consider the case where $f_t$ is
  differentiable.\footnote{This ensures both $g'_t$ and $\phi_t$ are
    uniquely determined; the proof still holds for general convex
    $f_t$, but only the sum $g'_t + \phi_t$ will be uniquely
    determined.}  By applying Theorem~\ref{thm:std} to
  Eq.~\eqref{eq:ftprlcomid} (taking $\Phi$ to be all the terms other
  than the cumulative regularization), there exists a $\phi_t \in \partial (\alpha_t \Psi)(x\ti)$  such that $g'_t = \grad f_t(\xti)$ and
  \begin{equation}\label{eq:fopt}
    g'_{1:t} + \phi_{1:t} + \grad \Rp_{1:t}(\xti) = 0.
  \end{equation}
  Similarly, applying Theorem~\ref{thm:std} to Eq.~\eqref{eq:comidF}
  implies there exists a $\hphi_t \in \partial (\alpha_t \Psi)(\hx\ti)$ such that
  $\hgp_t = \grad f_t(\hx\ti)$ and 
  \begin{equation}\label{eq:copt}
    \hgp_t + \hphi_t + \grad \Rp_{1:t}(\hx\ti) - \grad \Rp_{1:t}(\hx_t) = 0,
  \end{equation}
  recalling that $\grad_u B_R(u, v) = \grad R(u) - \grad R(v)$.

  We now proceed by induction on $t$, with the induction hypothesis
  that $x_t = \hx_t$.  The base case $t=1$ follows from the assumption
  that $\hx_1 = x_1 = 0$.  Suppose the induction hypothesis holds for
  $t$.  
  Taking Eq.~\eqref{eq:fopt} for $t-1$ gives $g'_{1:t-1} +
  \phi_{1:t-1} + \grad \Rp_{1:t-1}(x_t) = 0$, and since $\grad
  \Rp_t(x_t) = 0$, we have
  \begin{equation}\label{eq:tmopt}
    -\grad \Rp_{1:t}(x_t) =  g'_{1:t-1} + \phi_{1:t-1}
  \end{equation}

  Beginning from Eq.~\eqref{eq:copt},
  \begin{align}
   \hgp_t + \hphi_t + &\grad \Rp_{1:t}(\hx\ti) - \grad \Rp_{1:t}(\hx_t) \notag \\
       &= \hgp_t + \hphi_t + \grad \Rp_{1:t}(\hx\ti) - \grad \Rp_{1:t}(x_t) && \text{by the I.H.} \notag \\
       &= g'_{1:t-1} + \hgp_t + \phi_{1:t-1} + \hphi_t + \grad \Rp_{1:t}(\hx\ti),  \label{eq:otheropt}
  \end{align}
  where the last line uses Eq.~\eqref{eq:tmopt}.  The proof follows by
  applying Lemma~\ref{thm:std} to Eqs.~\eqref{eq:fopt} and
  \eqref{eq:otheropt}, and considering the pairs $(\hgp_t + \hphi_t,
  \hx\ti)$ and $(g'_t + \phi_t, \xti)$.  The equality $\phi_t =
  \hphi_t$ follows from the fact that $g'_t = \hgp_t$ since $f_t$ is
  differentiable.
\end{myproof}

\subsection{An Equivalence Theorem for 
Origin-Centered Regularization}\label{sec:gd-fprime}

For the moment, suppose $\Psi(x) = 0$.  So far, we have shown
conditions under which gradient descent on $f_t(x) = g_t \cdot x$ with
an adaptive step size is equivalent to
follow-the-proximally-regularized-leader.  In this section, we show
that mirror descent on the \emph{regularized} functions $\hf_t(x) =
g_t \cdot x + R_t(x)$, with a certain natural step-size, is equivalent
to a follow-the-regularized-leader algorithm with origin-centered
regularization.  For simplicity, in this section we restrict our
attention to linear $f_t$ (equivalently, non-implicit updates).  The
extension to implicit updates is straightforward.

The algorithm schema we consider next was introduced by \citet[Theorem
2.1]{bartlett07adaptive}.  Letting $R_t(x) =
\frac{\sigma_t}{2}\smnorm{x}_2^2$ and fixing $\eta_t =
\frac{1}{\sigma_{1:t}}$, their adaptive online gradient descent
algorithm is
\begin{equation*}
x_{t+1} = x_t - \eta_t \grad \hf_t(x_t)  
        = x_t - \eta_t (g_t + \sigma_t x_t)).
\end{equation*}
We show (in Corollary~\ref{cor:origin}) that this algorithm is
identical to follow-the-leader on the functions $\hf_t(x) = g_t\cdot x
+ R_t(x)$, an algorithm that is minimax optimal in terms of regret
against quadratic functions like $\hf$~\cite{abernethy08}.  As with
the previous theorem, the difference between the two is how they
handle an arbitrary $\Psi$.  If one uses $\Rp_t(x) =
\frac{\sigma_t}{2}\smnorm{x - x_t}^2_2$ in place of $R_t(x)$, this
algorithm reduces to standard online gradient descent
\citep{do09proximal}.

The key observation of~\citet{bartlett07adaptive} is that if the
underlying functions $f_t$ have strong convexity, we can roll that
into the $R_t$ functions, and so introduce less additional
stabilizing regularization, leading to regret bounds that interpolate
between $\sqrt{T}$ for linear functions and $\log T$ for strongly
convex functions.  Their work did not consider composite objectives
($\Psi$ terms), but our equivalence theorems show their adaptivity
techniques can be lifted to algorithms like \orgn and \prox that
handle such non-smooth functions more effectively than mirror descent
formulations.

We will prove our equivalence theorem for a generalized versions of
the algorithm.  Instead of vanilla gradient descent, we analyze the
mirror descent algorithm of Equation~\eqref{eq:gen-comid}, but now
$g_t$ is replaced by $\grad \hf_t(x_t)$, and we add the composite term
$\Psi(x)$.

\begin{theorem} \label{thm:gen-gdfprime}
Let $f_t(x) = g_t \cdot x$, and let $\hf_t(x) = g_t \cdot x + R_t(x)$,
where $R_t$ is a differentiable convex function.  Let $\Psi$ be an
arbitrary convex function.  Consider the composite-objective
mirror-descent algorithm which plays
\begin{equation}\label{eq:mdhf}
\hx_{t+1} = \argmin_x \grad \hf_t(\hx_t) \cdot x + \Psi(x) + \Brtt(x, \hx_t),
\end{equation}
and the FTRL algorithm which plays
\begin{equation}\label{eq:ftrl}
 x_{t+1} = \argmin_x \hf_{1:t}(x) + \phi_{1:t-1} \cdot x + \Psi(x),
\end{equation}
for $\phi_t \in \partial \Psi(\xti)$ such that
$g_{1:t} + \grad R_{1:t}(\xti) + \phi_{1:t-1} + \phi_t = 0$.
If both algorithms play $\hx_1 = x_1 = 0$, then they are equivalent,
in that $x_t = \hx_t$ for all $t > 0$.
\end{theorem}

The most important corollary of this result is that it lets us add the
adaptive online gradient descent algorithm to
Table~\ref{table:algorithms}.  It is also instructive to specialize to the
simplest case when $\Psi(x) = 0$ and the regularization is quadratic:
\begin{corollary}\label{cor:origin}
Let $f_t(x) = g_t \cdot  x$ and 
$\hf_t(x) = g_t \cdot x + \frac{\sigma_t}{2}\smnorm{x}^2_2.$
Then the following algorithms play identical points:
\begin{itemize} \itemsep -2pt
   \item FTRL, which plays $x_{t+1} = \argmin_{x} \hf_{1:t}(x).$ 

   \item Gradient descent on the functions $\hf$ using the step size
   $\eta_t = \frac{1}{\sigma_{1:t}}$, which plays 
    \[ x_{t+1} = x_t - \eta_t \grad \hf_t(x_t) \]

   \item Revisionist constant-step size gradient descent with $\eta_t
   = \frac{1}{\sigma_{1:t}}$, which plays 
   \[ x_{t+1} = -\eta_t g_{1:t}.\]
\end{itemize}
\end{corollary}

The last equivalence in the corollary follows from deriving the closed
form for the point played by FTRL.  We now proceed to the proof of the
general theorem:

\begin{myproof}{Proof of Theorem~\ref{thm:gen-gdfprime}}
The proof is by induction, using the induction hypothesis $\hx_t =
x_t$.  The base case for $t=1$ follows by inspection.  Suppose the
induction hypothesis holds for $t$; we will show it also holds for
$t+1$.  
Again let $r_t = \grad R_t$ and consider Equation~\eqref{eq:ftrl}.
Since $R_1$ is assumed to be strongly convex, applying
Theorem~\ref{thm:std} gives us that $x_t$ is the unique solution to
$\grad \hf_{1:t-1}(x_t) + \phi_{1:t-1}= 0$ and so $g_{1:t-1} +
r_{1:t-1}(x_t) +  \phi_{1:t-1} = 0$.  Then, by the induction hypothesis,
\begin{equation}\label{eq:ihg}
- r_{1:t-1}(\hx_t) = g_{1:t-1} + \phi_{1:t-1}.
\end{equation}

Now consider Equation~\eqref{eq:mdhf}.  Since $R_1$ is strongly
convex, $\Brtt(x, \hx_t)$ is strongly convex in its first argument,
and so by Theorem~\ref{thm:std} we have that $\hx_{t+1}$ and some
$\phi_t' \in \partial
\Psi(\hx_{t+1})$ are the unique solution to
\[
 \grad \hf_t(\hx_t) + \phi_t' + \rtt(\hx_{t+1}) - \rtt(\hx_t) = 0,
\]
since $\grad_p \Br_R(p,q) = r(p) - r(q)$.
Beginning from this equation,
\begin{align*}
0 
  &= \grad \hf_t(\hx_t) + \phi_t' + \rtt(\hx_{t+1}) - \rtt(\hx_t) \\
  &= g_t + r_t(\hx_t) + \phi_t' + \rtt(\hx_{t+1}) - \rtt(\hx_t) \\
  &= g_t + \rtt(\hx_{t+1}) + \phi_t'  - r_{1:t-1}(\hx_t) \\
  &= g_t + \rtt(\hx_{t+1}) + \phi_t' + g_{1:t-1} + \phi_{1:t-1} && \text{Eq~\eqref{eq:ihg}}\\
  &= g_{1:t} + \rtt(\hx_{t+1}) + \phi_{1:t-1} + \phi'_t.
\end{align*}
 Applying Theorem~\ref{thm:std} to Equation~\eqref{eq:ftrl}, $(\xti,
 \phi_t)$ are the unique pair such that
\[ g_{1:t} + \rtt(x_{t+1}) + \phi_{1:t-1} + \phi_t = 0\]
and $\phi_t \in \partial \Psi(x_{t+1})$,
and so we conclude $\hx_{t+1} = \xti$ and $\phi_t' = \phi_t$.
\end{myproof}

\section{Regret Analysis}\label{sec:analysis}

In this section, we prove the regret bounds of
Theorem~\ref{thm:genmain} and Corollary~\ref{cor:main}.  Recall the
general update we analyze is
\begin{equation} \tag{\ref{eq:update}}
\xti = \argmin_x  g'_{1:t-1}\cdot x + f_t(x) + \alpha_{1:t}\Psi(x) + R_{1:t}(x) 
\end{equation}
where $\gpt \in \partial f_t(\xti)$.
It will be useful to consider the equivalent (by
Theorem~\ref{thm:std}) update
\begin{equation}\label{eq:primeupdate}
x_{t+1} = \argmin_x g'_{1:t}\cdot x + \alpha_{1:t} \Psi(x) + R_{1:t}(x).
\end{equation}
We can view this alternative update as running FTRL on the
linear approximations of $f_t$ taken at $x_{t+1}$,
\[\fln_t(x) = f_t(x_{t+1}) + \gpt \cdot (x - \xti).\]
To see the equivalence, note the constant terms in $\fln$ change
neither the argmin nor regret.  This is still an implicit update, as
implementing the update requires an oracle to compute an appropriate
subgradient $\gpt$ (say, by finding $x_{t+1}$ via
Equation~\eqref{eq:update}).

This re-interpretation is essential, as it lets us analyze a
follow-the-leader algorithm on convex functions; note that the
objective function of Equation~\eqref{eq:update} is not the sum of one
convex function per round, as when moving from $x_{t-1}$ to $x_t$ we
effectively add $g'_{t-1}\cdot x - f_{t-1}(x) + f_t(x)$ to the
objective, which is not in general convex.  By immediately applying an
appropriate linearization of the loss functions, we avoid this
non-convexity.

The affine functions $\fln$ lower bound $f_t$, and so can be used to
lower bound the loss of any $\xs$; however, in contrast to the more
typical subgradient approximations taken at $x_t$, these linear
functions are not tight at $x_t$, and so our analysis must also
account for the additional loss $f_t(x_t) - \fln_t(x_t)$.  
Before formalizing these arguments in the proof of
Theorem~\ref{thm:genmain}, we prove the following lemma.  We will use
this lemma to get a tight bound on the regret of the algorithm against
the linearized functions $\fln$, but it is in fact much more general.

\begin{lemma}[Strong FTRL Lemma]\label{lem:tight_ftrl}
  Let $f_t$ be a sequence of arbitrary (e.g., non-convex) loss
  functions, and let $R_t$ be arbitrary non-negative regularization
  functions.  Define $\hf_t(x) = f_t(x) + R_t(x)$.  Then, if we play
  $x_{t+1} = \argmin_x \hf_{1:t}(x)$, our regret against the functions
  $f_t$ versus an arbitrary point $\xs$ is bounded by
  \[
    \Regret \leq 
       R_{1:T}(\xs) + \sum_{t=1}^T \Big(\hf_{1:t}(x_t) 
          - \hf_{1:t}(\xti) - R_t(x_t)\Big).
  \]
\end{lemma} 
A weaker (though sometimes easier to use) version of this lemma,
stating
\[ \Regret \leq R_{1:T}(\xs) + \sum_{t=1}^T \big(f_{t}(x_t) - f_{t}(\xti)\big),\]
has been used
previously~\cite{kalai03ftpl,hazan08extract,mcmahan10boundopt}.  In
the case of linear functions with quadratic regularization, as in the
analysis of ~\emcite{mcmahan10boundopt}, the weaker version loses a
factor of $\h$ (corresponding to a $\sqrt{2}$ in the final bound).
The key is that in that case, being the leader is \emph{strictly
  better} than playing the post-hoc optimal point.  Quantifying this
difference leads to the improved bounds for FTPRL in this paper.

\begin{myproof}{Proof of Lemma~\ref{lem:tight_ftrl}}
First, we consider regret against the functions $\hf$ for not playing $\xs$:
\begin{align*}
\Regret(\hf) 
   &= \sum_{t=1}^T (\hf_t(x_t) - \hf_t(\xs))  &&\text{by definition}\\
   &= \sum_{t=1}^T \hf_t(x_t) - \hf_{1:T} (\xs) \\
   &= \sum_{t=1}^T (\hf_{1:t}(x_t) - \hf_{1:t-1}(x_{t})) -\hf_{1:T}(\xs)
        && \text{where $f_{1:0}(x) = 0$} \\
   &\leq \sum_{t=1}^T (\hf_{1:t}(x_t) - \hf_{1:t-1}(x_{t})) -\hf_{1:T}(x_{T+1})
        && \text{since $x_{T+1}$ minimizes $\hf_{1:T}$}\\
   &= \sum_{t=1}^T (\hf_{1:t}(x_t) - \hf_{1:t}(x_{t+1})),
\end{align*}
where the last line follows by simply re-indexing the $-\hf_{1:t}$
terms.  Equivalently, applying the definitions of regret and $\hf$,
\[ 
  \sum_{t=1}^T (f_t(x_t) + R_t(x_t)) - f_{1:T}(\xs) - R_{1:T}(\xs)
     \leq \sum_{t=1}^T (\hf_{1:t}(x_t) - \hf_{1:t}(x_{t+1})).
\]
Re-arranging the inequality proves the theorem.
\end{myproof}

With this lemma in hand, we turn to our main proof.  It is worth
noting that the second half of the proof simplifies significantly when
we choosing $x_t = y_t$, as in FTPRL.

\begin{myproof}{Proof of Theorem~\ref{thm:genmain}}
Recall $\fln_t(x) = f_t(x_{t+1}) + \gpt \cdot (x - \xti)$, a linear
approximation of $f_t$ taken at the next point, $x_{t+1}$.
We can bound the regret of our algorithm (expressed as an FTRL
algorithm on the functions $\fln_t$, Equation~\eqref{eq:primeupdate})
against the functions $\fln_t$ by applying Lemma~\ref{lem:tight_ftrl}
to the functions $\fln_t$ with regularization functions $R'_t(x) = R_t(x) +
\alpha_t \Psi(x)$.  Because we are taking the linear approximation at
$\xti$ instead of $x_t$, it may be the case that our actual loss
$f_t(x_t)$ on round $t$ is greater than the loss under $\fln_t$, that
is we may have $f_t(x_t) > \fln_t(x_t)$.  Thus, we must account for
this additional regret.  From the definition of regret we have
\begin{align*} 
\Regret(f)   
 &= \Regret(\fln) + \sum_{t=1}^T (f_t(x_t) - \fln_t(x_t))
    + ( \fln_{1:t}(\xs) - f_{1:t}(\xs) ) \\
 &\leq \Regret(\fln) + \sum_{t=1}^T (f_t(x_t) - \fln_t(x_t)) 
\intertext{since  $\fln_t$ lower bounds $f_t$, 
and letting $\flnR_t(x) = \fln_t(x) + R'_t(x)$,}
 &\leq \underbrace{R'_{1:T}(\xs) 
       + \sum_{t=1}^T (\flnR_{1:t}(x_t) - \flnR_{1:t}(\xti) - R_t'(x_t)) 
       }_{\text{Lemma~\ref{lem:tight_ftrl} on $\fln_t$ and $R'_t$}}
     \quad + \underbrace{\sum_{t=1}^T (f_t(x_t) - \fln_t(x_{t})).}_{
          \text{Underestimate of real loss at $x_t$}}
\end{align*}
Let $\Delta_t$ be the contribution of the non-regularization terms for a particular $t$,
\begin{align*}
\Delta_t
  & = \flnR_{1:t}(x_t) - \flnR_{1:t}(\xti) + f_t(x_t) - \fln_t(x_{t}), \\
  &= \fln_{1:t}(x_t) + R'_{1:t}(x_t) - \fln_{1:t}(\xti) - R'_{1:t}(\xti) +  f_t(x_t) - \fln_t(x_{t}),\\
  &= \fln_{1:t-1}(x_t) + R'_{1:t}(x_t) - \fln_{1:t}(\xti) - R'_{1:t}(\xti) +  f_t(x_t) ,\\
  &= (\fln_{1:t-1}(x_t) +  R'_{1:t}(x_t)  + f_t(x_t)) - (\fln_{1:t}(\xti) + R'_{1:t}(\xti) ). 
\end{align*}
For the terms containing $\xti$, using the fact that $\fln_t(\xti) = f(\xti)$,
we have
\begin{equation} \label{eq:xti_as_h}
\fln_{1:t}(\xti) + R'_{1:t}(\xti)  = \fln_{1:t-1}(\xti) + R'_{1:t}(\xti)  + f_t(\xti).
\end{equation}
For a fixed $t$, we define two helper functions $\hha$ and $\hhb$.  Let 
\[ \hhb(x) = \fln_{1:t-1}(x) + R_{1:t}(x) + \alpha_{1:t} \Psi(x) + f_t(x),\] 
so
$\Delta_t =  \hhb(x_t) - \hhb(\xti).$
Define
\begin{align*}
 \hha(x) &= \fln_{1:t-1}(x) + R_{1:t-1}(x) + \alpha_{1:t-1} \Psi(x).
\end{align*}
Then we can write
\[\hhb(x) = \hha(x) + f_t(x) + R_t(x) + \alpha_t \Psi(x).\]
By definition of our updates, $\xt = \argmin_x \hha(x)$
(using Eq.~\eqref{eq:primeupdate}) and $\xti = \argmin_x \hhb(x)$.

Now, suppose we choose regularization $R_t(x) = \h\smnorm{Q_t^\h (x -
  y_t)}^2$.  The remainder of the proof is accomplished by bounding
$\hhb(x_t) - \hhb(\xti)$, with the aid of two lemmas (stated and
proved below).  First, by expanding $\hha$ and dropping constant terms
(which cancel from $\Delta_t$), we have
\begin{align*}
\hha(x) 
   &= \h x^\tp Q_{1:t-1} x  + \Big(g'_{1:t-1} - \h \sum_{s=1}^{t-1} Q_s y_s\Big) \cdot x  + \alpha_{1:t-1}\Psi(x)   \\
   &= \hnorm{Q_{1:t-1}^\h (x - \xt)}^2 + \hp(x) + k_t' && \text{Lemma~\ref{lem:rewrite}}
\end{align*}
for some constant $k_t' \in \R$.  Recall $Q_{1:t-1}^\h = (Q_1 + \dots
+ Q_{t-1})^\h$.  Now, we can apply Lemma~\ref{lem:newQ}.  The constant
$k_t'$ cancels out, and we take $\Qa = Q_{1:t-1}$, $\Qb = Q_t$, $\hpa
= \hp$, $\hpb = \alpha_t \Psi$, $\pxa = x_t$, etc.  Thus, letting $d_t
= y_t - x_t$,
\begin{align}
\Delta_t &= \hha(x_t) - \hhb(\xti) \notag \\
 &\leq (g_t - \h \wg_t)^\tp \Qtt^{-1} \wg_t 
   +\hnorm{\Qtt\invh(Q_t d_t)}^2 - g \Qtt\inv Q_t d_t
   + \alpha_t \Psi(x_t) - \alpha_t \Psi(\xti).
\label{eq:dtb}
\end{align}

We now re-incorporate the $-R'_t(x_t)$ terms not included in the
definition of $\Delta_t$.  Note $R'_t(x) \geq R_t(x)$, and $R_t(x_t) =
\hnorm{Q_t^\h d_t}^2$.  Then
\begin{align*}
\hnorm{\Qtt\invh (Q_t d_t)}^2 - \hnorm{Q_t^\h d_t}^2 
&= \h d_t^\tp Q_t^\tp \Qtt\inv Q_t d_t - \h d_t^\tp Q_t d_t \\
&\leq \h d_t^\tp Q_t^\tp Q_t\inv Q_t d_t - \h d_t^\tp Q_t d_t = 0
\end{align*}
where we have used the fact that $\Qtt \mge Q_t \mge 0$ implies
$Q_t\inv \mge \Qtt\inv \mge 0$.  Combining this result with
Eq.~\eqref{eq:dtb} and adding back the $R_{1:t}(\xs)$ term gives
\[ \Regret \leq R_{1:t}(\xs) 
      + \sum_{t=1}^T  \Big((g_t - \h \wg_t)^\tp \Qtt^{-1} \wg_t  - g_t \Qtt\inv Q_t (y_t - x_t)
      +  \alpha_t \Psi(x_t) - \alpha_t \Psi(\xti)\Big).
\]
Defining $\alpha_{T+1} = 0$, observe
\begin{align*}
\sum_{t=1}^T \alpha_t \Psi(x_t) - \alpha_t \Psi(\xti) 
   & = \sum_{t=1}^T\alpha_t \Psi(x_t) - \alpha_{t+1} \Psi(\xti) + (\alpha_{t+1} - \alpha_{t}) \Psi(\xti)  \\
   & \leq \sum_{t=1}^T\alpha_t \Psi(x_t) - \alpha_{t+1} \Psi(\xti)  \\
   & = \alpha_1 \Psi(x_1) - \alpha_{T+1} \Psi(x_{T+1})  = 0 
\end{align*}
where the inequality uses the fact that $0 \leq \alpha_{t+1} \leq
\alpha_t$ and $\Psi(x) \geq 0$.  The last equality follows from
$\Psi(x_1) = \Psi(0) = 0$ and $\alpha_{T+1}=0$.  Thus we conclude
\[ \Regret \leq R_{1:t}(\xs) 
      + \sum_{t=1}^T (g_t - \h \wg_t)^\tp \Qtt^{-1} \wg_t  - g_t^\tp \Qtt\inv Q_t (y_t - x_t).
\]

The second inequality in the theorem statement follows from
Equation~\eqref{eq:implbtr} of Lemma~\ref{lem:newQ}.
\end{myproof} 

Only $R_{1:t}(\xs)$ and the last term in the bound depend on the
center of the regularization $y_t$; the final term can either increase
or decrease regret, depending on the relationship between $g_t$ and
$y_t - x_t$ (note $g_t$ is not known when $y_t$ is selected).  If we
consider the simple case where all $Q_t = \sigma_t I$, observe that if
$-g_t \cdot(y_t - x_t) > 0$ then (roughly speaking) both the new
regularization penalty and the gradient of the loss function are
pulling $x_{t+1}$ away from $x_t$ in the same direction, and so regret
from this term will be larger.

\begin{myproof}{Proof of Corollary~\ref{cor:main}}
  We first consider FTPRL.   Let $Q_t = \sigma_tI$, and define $\sigma_t$ such that
  $\sigma_{1:t} = \ifrac{G\sqrt{2t}}{D}$.
  Then taking Theorem~\ref{thm:genmain} with $x_t = y_t$ gives
  \begin{align*}
    \Regret 
      &\leq \sum_{t=1}^T \frac{\sigma_t}{2} \norm{\xs - x_t}^2 + \sum_{t=1}^T \frac{{g_t}^2}{2 \sigma_{1:t}} \\
      &\leq \frac{\sigma_{1:T}}{2}D^2 + \sum_{t=1}^T \frac{G^2}{2 \sigma_{1:t}} \\
      &= \frac{GD\sqrt{2T}}{2} + \frac{GD}{2 \sqrt{2}}\sum_{t=1}^T \frac{1}{\sqrt{t}} \\
      &\leq DG\sqrt{2T},
  \end{align*}
  where the last inequality uses the fact that $\sum_{t=1}^T\frac{1}{\sqrt{t}} \leq 2\sqrt{T}$. 

  Recall the characterization of implicit-update mirror descent from
  Section~\ref{sec:equiv}.  Thus, in this case we have $\fu_t(x)
  \leftarrow \fw_t(x) + \IF(x)$.  Let $g^w_t = \grad \fw_t(\xt)$, so
  in the analysis we have $g_t^u = g^w_t + \phi_t$.  Following standard
  arguments, e.g.~\cite{bartlett07adaptive,duchi10composite}, it is
  straightforward to use the Pythagorean theorem for Bregman
  divergences to show $\hnorm{g^w_t}^2 \geq \hnorm{g^u_t}^2$, and then the
  result follows as for FTPRL.

  For regularized dual averaging we have $y_t = 0$.  Again let $Q_t =
  \sigma_tI$, and define $\sigma_t$ such that $\sigma_{1:t} =
  \ifrac{2G\sqrt{2t}}{D}.$ Then, Theorem~\ref{thm:genmain} gives
  \[\Regret 
      \leq \sum_{t=1}^T \frac{\sigma_t}{2} \norm{\xs}^2 + \sum_{t=1}^T \frac{{g_t}^2}{2 \sigma_{1:t}} 
           - \frac{\sigma_t}{\sigma_{1:t}} g_t \cdot x_t.
  \]
  The proof is largely similar to that for FTPRL, but we must deal
  with an extra term.  First, note for $t \geq 2$,
  \[ \sigma_t = \sigma_{1:t} - \sigma_{1:t-1} 
  = \frac{2G\sqrt{2}}{D}(\sqrt{t} - \sqrt{t-1})
  \leq \frac{2G\sqrt{2}}{D}\left(\frac{1}{2\sqrt{t-1}}\right)
  \leq \frac{2G}{D\sqrt{t}},
  \]
  where we have used $\sqrt{t} - \sqrt{t-1} \leq \frac{1}{2\sqrt{t-1}}$ and for $t \geq 2$, $1/\sqrt{t-1} \leq \sqrt{2}/{\sqrt{t}}$.
  Then, noting the term for $t=1$ is zero since $x_1 = 0$,
  \begin{align*}
    \sum_{t=1}^T-\frac{\sigma_t}{\sigma_{1:t}} g_t \cdot x_t 
    \leq  G D \sum_{t=2}^T\frac{\sigma_t}{\sigma_{1:t}} 
    \leq  G D \sum_{t=2}^T\frac{D}{2G\sqrt{2t}}  \frac{2G}{D\sqrt{t}} 
    \leq  \frac{G D}{\sqrt{2}} \sum_{t=2}^T\frac{1}{t}  
    \leq  \frac{G D}{\sqrt{2}} (\ln T + 1).  
  \end{align*}
Applying this observation, 
  \begin{align*}
    \Regret 
      &\leq \sum_{t=1}^T \frac{\sigma_t}{2} \norm{\xs}^2 + \sum_{t=1}^T \frac{{g_t}^2}{2 \sigma_{1:t}} 
         - \frac{\sigma_t}{\sigma_{1:t}} g_t \cdot x_t \\
      &\leq \frac{\sigma_{1:T}}{2} \left(\frac{D}{2}\right)^2 
         + \sum_{t=1}^T \frac{G^2}{2 \sigma_{1:t}} +  \frac{G D}{\sqrt{2}} \ln T + \BO(1) \\
      &= \frac{GD\sqrt{2T}}{4} + \frac{GD}{2 \sqrt{2}}\sum_{t=1}^T \frac{1}{\sqrt{t}} 
         +  \frac{G D}{\sqrt{2}} \ln T + \BO(1) \\
      &\leq \frac{1}{2}DG\sqrt{2T} +  \frac{G D}{\sqrt{2}} \ln T + \BO(1).
  \end{align*}
\end{myproof}
We now prove the two lemmas used in bounding the $\hha(x_t) -
\hhb(\xti)$ terms in the proof of Theorem~\ref{thm:genmain}.

\begin{lemma}\label{lem:rewrite}
Let $\Psi$ be a convex function defined on $\R^n$, and let $Q \in
\Snpp$.  Define
\[ h(x) = \frac{1}{2} x^\tp Q x + b \cdot x + \Psi(x),\]
and let $x^* = \argmin_x h(x)$.
 Then, we can rewrite $h$ as
\[ h(x) = \frac{1}{2}\norm{Q^\h(x - x^*)}^2 + \hp(x) + k,\]
where $k \in \R$ and $\hp$ is convex with $0 \in \partial\hp(x^*)$.
\end{lemma}

\begin{proof}
  Since $Q\in \Snpp$, the function $\frac{1}{2} x^\tp Q x$ is strongly
  convex, and so using Theorem~\ref{thm:std}, $h$ has a unique
  minimizer $x^*$ and there exists a (unique) $\phi$ such that
  \begin{equation} \label{eq:phidef}
    Qx^* + b + \phi = 0
  \end{equation}
  with $\phi \in \partial \Psi(x^*)$.  Define $\hp(x) = \Psi(x) - \phi
  \cdot x$, and note $0 \in \partial \hp(x^*)$.  Then,
  \begin{align*}
    h(x) 
    =& \frac{1}{2} x^\tp Q x + b \cdot x + \Psi(x)\\
    =& \frac{1}{2} x^\tp Q x + (b + \phi) \cdot x + \hp(x) 
           && \text{Defn. $\hp(x)$}\\
    =& \frac{1}{2} x^\tp Q x - x^\tp Q x^* + \hp(x)
    && \text{Eq.~\eqref{eq:phidef}} \\
    =& \frac{1}{2} \norm{Q^{\frac{1}{2}}(x - x^*)}^2 + \hp(x)
    - \frac{1}{2} \norm{Q^{\frac{1}{2}}x^*}^2, 
  \end{align*}
  where $\hp$ and $k = -\frac{1}{2} \norm{Q^{\frac{1}{2}}x^*}^2$
  satisfy the requirements of the theorem.
\end{proof}

\begin{lemma}\label{lem:newQ}
  Let $\xa \in \R^n$, let $\hpa$ be a convex function such that $0
  \in \partial \hpa(\xa)$, and let $\Qa \in \Snpp$.  Define
\[ 
  \hha(x) = \hnorm{\Qah (x - \xa)}^2 + \hpa(x),
\]
so $\xa = \argmin_x \hha(x)$.  Let $f$ and $\hpb$ be convex functions,
let $\Qb \in \Snp$, and define
\[ 
  \hhb(x) = \hha(x) + f(x) + \hnorm{\Qbh(x - \yc)}^2 + \hpb(x).
\] 
Let $\xb = \argmin_x \hhb(x)$, let $g \in \partial f(\xa)$, let $d =
\yc - \xa$, and let $\Qab = \Qa + \Qb$.  Then, there exists a certain
subgradient $\wg$ of $f$ such that
\begin{equation} \label{eq:pdc}
 \hhb(\pxa) - \hhb(\pxb) \leq 
  \big(g - \h \wg\big)^\tp \Qab\inv \wg  
    + \hnorm{\Qab\invh(\Qb d)}^2 - g^\tp \Qab\inv\Qb d
    + \hpb(\pxa) - \hpb(\pxb) 
\end{equation}
Further, 
\begin{equation} \label{eq:implbtr}
  \big(g - \h \wg\big)^\tp \Qab\inv \wg \leq \h g^\tp \Qab\inv g  - \delta 
\end{equation}
where $\delta \geq 0$.
\end{lemma}

As we will see in the proof, $\delta > 0$ when the implicit update is
non-trivial.

\begin{proof}
  To obtain these bounds, we first analyze the problem without the
  $\Phi$ terms.  For this purpose, we define
  \[
  \wh_2(x) = \hnorm{\Qah (x - \xa)}^2 + \hnorm{\Qbh(x - \yc)}^2 +  f(x),
  \]
  and let  $\wx_2 = \argmin_x \wh_2(x)$.
  We can re-write
  \begin{align*}
    \wh_2(x) 
    &= f(x) + \hnorm{\Qah(x - \xa)}^2  + \hnorm{\Qbh(x - \xa - d)}^2 \\
    &= f(x) + \hnorm{\Qabh (x - \xa)}^2   - d^\tp \Qb (x - \xa) + \hnorm{\Qbh d}^2.
  \end{align*}
  Then, using Theorem~\ref{thm:std} on the last expression, there exists a $\wg \in \partial f(\wx_2)$ such that
  $\wg + \Qab(\wx_2 - \xa) - \Qb d = 0,$
  and so in particular
  \begin{equation}\label{eq:optz}
    \wx_2 - \xa = \Qab\inv (\Qb d -\wg).
  \end{equation}
  Then,
  \begin{align*}
    \wh_2(\xa)& - \wh_2(\wx_2) \\
    &= f(\xa) + \hnorm{\Qbh d}^2  - f(\wx_2) - \hnorm{\Qabh (\wx_2 - \xa)}^2   + d^\tp \Qb(\wx_2 - \xa) - \hnorm{\Qbh d}^2 \\
    &= f(\xa) - f(\wx_2) - \hnorm{\Qabh (\wx_2 - \xa)}^2  + d^\tp \Qb(\wx_2 - \xa),
    \intertext{and since $f(\wx_2) \geq f(\xa) + g(\wx_2 - \xa)$ implies $f(\xa) - f(\wx_2) \leq  -g (\wx_2 - \xa)$,}
    &\leq - \hnorm{\Qabh (\wx_2 - \xa)}^2  + (\Qb d - g)^\tp (\wx_2 - \xa)  \\
    \intertext{and applying Eq.~\eqref{eq:optz},}
    &= - \hnorm{\Qab\invh (\Qb d -\wg)}^2  + (\Qb d - g)^\tp (\Qab\inv (\Qb d -\wg))  \\
    &= g^\tp \Qab\inv \wg - \hnorm{\Qab\invh \wg}^2 +\hnorm{\Qab\invh(\Qb d)}^2 - g^\tp \Qab\inv\Qb d,\\
  \end{align*}
  and so we conclude 
  \begin{equation}\label{eq:whb}
    \wh_2(\xa) - \wh_2(\wx_2) \leq \big(g - \h \wg)^\tp \Qab\inv \wg  +\hnorm{\Qab\invh(\Qb d)}^2 - g^\tp \Qab\inv\Qb d.
  \end{equation}

  Next, we quantify the advantage offered by implicit updates.
  Suppose we choose $\wx_2$ by optimizing a version of $\wh_2$ where $f$ is linearized at $x_1$:
  \[
     \lh_2(x) = \hnorm{\Qah (x - \xa)}^2 + \hnorm{\Qbh(x - \yc)}^2 + g\cdot x.
  \]
  Let $\lx_2 = \argmin_x \lh_2(x)$.  We say the implicit update is
  non-trivial when $\wh_2(\wx_2) < \wh_2(\lx_2)$, that is, the
  implicit update provides a better solution to the optimization
  problem defined by $\wh_2$.  By definition $\wh_2(\wx_2) \leq
  \wh_2(\lx_2)$, and we can write $\wh_2(\wx_2) = \wh_2(\lx_2) -
  2\delta$ with $\delta \geq 0$.  Let $R_1(x) = \hnorm{\Qah (x -
    \xa)}^2$ and $R_2(x) = \hnorm{\Qbh(x - \yc)}^2$.  Then, by the
  definition of $\wx_2$ and $\lx_2$ we have
  \begin{align}
    R_{1:2}(\lx_2) + g \cdot \lx_2 &\leq R_{1:2}(\wx_2) + g \cdot \wx_2 \notag \\
    R_{1:2}(\wx_2) + \wg \cdot \wx_2 &= 
     R_{1:2}(\lx_2) + \wg \cdot \lx_2 - 2\delta \notag \\
    \intertext{and adding and canceling terms common to both sides gives}
    g \cdot \lx_2 + \wg \cdot \wx_2 &\leq 
        g \cdot \wx_2 + \wg \cdot \lx_2 - 2\delta. \label{eq:gwwl}
  \end{align}
  Following Equation~\eqref{eq:optz} $\wx_2 = \Qab\inv (\Qb d -\wg) +
  \xa$ or $\wx_2 = -\Qab\inv\wg + \kappa$ where $\kappa = \Qab\inv \Qb
  d + \xa$.  Similarly, $\lx_2 = -\Qab\inv g + \kappa$.  Plugging into
  Equation~\eqref{eq:gwwl}, and noting the $\kappa$ terms cancel, we have
  \[
    -g^\tp \Qab\inv g  - \wg^\tp \Qab\inv\wg \leq 
       -g^\tp \Qab\inv\wg - \wg^\tp \Qab\inv g  - 2\delta,
  \]
  or re-arranging and dividing by one-half,
  \begin{equation} \label{eq:iadv}
    \h g^\tp \Qab\inv g - \delta \geq (g - \h \wg)^\tp \Qab\inv\wg,
  \end{equation}

  We now consider the functions that include the $\Phi$ terms.  Note
  \[
     h_2(x_2) = \wh_2(x_2) + \hpa(x_2) + \hpb(x_2) 
         \geq \wh_2(\wx_2) + \hpa(x_1) + \hpb(x_2).
  \]
Then, 
\begin{align*}
 h_2(x_1) - h_2(x_2) 
    &= \wh_2(x_1) + \hpa(x_1) + \hpb(x_1) - h_2(x_2) \\
    &\leq \wh_2(x_1) + \hpa(x_1)  + \hpb(x_1) 
           - \wh_2(\wx_2) - \hpa(x_1) - \hpb(x_2)\\
    &= \wh_2(x_1) - \wh_2(\wx_2) + \hpb(x_1) - \hpb(x_2).
\end{align*}
Combining this fact with Equations~\eqref{eq:whb} and \eqref{eq:iadv}
proves the theorem.
\end{proof}

\section{Experiments with $L_1$ Regularization} \label{sec:exp}

We compare \fobos, \prox, and \orgn on a variety of datasets to
illustrate the key differences between the algorithms, from the point
of view of introducing sparsity with $L_1$ regularization.  
In all experiments we optimize log-loss (see Section~\ref{sec:algs}).
Since our goal here is to show the impact of the different choices of
regularization and the handling of the $L_1$ penalty, for simplicity
we use first-order updates rather than implicit updates for the
log-loss term.  

For an experimental evaluation of implicit updates, we refer the
reader to~\emcite{karampatziakis10importance}, which provides a
convincing demonstration of the advantages of implicit updates on both
importance weighted and standard learning problems.

\paragraph{Binary Classification}
We compare \prox, \orgn, and \fobos on several public datasets.  We
used four sentiment classification data sets (Books, Dvd, Electronics,
and Kitchen), available from~\cite{sentiment}, each with 1000 positive
examples and 1000 negative examples,\footnote{We used the features
  provided in processed\_acl.tar.gz, and scaled each vector of counts
  to unit length.  } as well as the scaled versions of the rcv1.binary 
(20,242 examples) and news20.binary (19,996 examples) data sets from
LIBSVM~\cite{libsvmdata}. 

All our algorithms use a learning rate scaling parameter $\gamma$ (see
Section~\ref{sec:algs}).  The optimal choice of this parameter can
vary somewhat from dataset to dataset, and for different settings of
the $L_1$ regularization strength $\lambda$.  For these experiments,
we first selected the best $\gamma$ for each (dataset, algorithm,
$\lambda$) combination on a random shuffling of the dataset.  We did
this by training a model using each possible setting of $\gamma$ from
a reasonable grid (12 points in the range $[0.3, 1.9])$, and choosing
the $\gamma$ with the highest online AUC.  We then fixed this value,
and report the average AUC over 5 different shufflings of each 
dataset.  We chose the area under the ROC curve (AUC) as our
accuracy metric as we found it to be more stable and have less
variance than the mistake fraction.  However, results for
classification accuracy were qualitatively very similar.

\begin{table*}[t!]
\begin{center}
\caption{AUC (area under the ROC curve) for online predictions and
  sparsity in parentheses.  The best value for each dataset is
  shown in bold.  For these experiments, $\lambda$ was fixed at $0.05/T$.}
\label{table}
\vspace{0.1in}
\begin{small}
\begin{sc}
\begin{tabular}{lcccccr}
\hline
          Data &                       FTRL-Proximal &                                 RDA &                               FOBOS \\
\hline
\hline
         books &                       0.874 (0.081) &     \textbf{0.878} (\textbf{0.079}) &                       0.877 (0.382) \\
           dvd &                       0.884 (0.078) &              0.886 (\textbf{0.075}) &              \textbf{0.887} (0.354) \\
   electronics &                       0.916 (0.114) &     \textbf{0.919} (\textbf{0.113}) &                       0.918 (0.399) \\
       kitchen &              0.931 (\textbf{0.129}) &              \textbf{0.934} (0.130) &                       0.933 (0.414) \\
          news &              0.989 (\textbf{0.052}) &              \textbf{0.991} (0.054) &                       0.990 (0.194) \\
          rcv1 &              0.991 (\textbf{0.319}) &              \textbf{0.991} (0.360) &                       0.991 (0.488) \\

\hline
web search ads &     \textbf{0.832} (\textbf{0.615}) &                       0.831 (0.632) &                       0.832 (0.849) \\
\hline
\end{tabular}
\end{sc}
\end{small}
\end{center}
\vspace{-0.1in}
\end{table*}

\paragraph{Ranking Search Ads by Click-Through-Rate}
We collected a dataset of about 1,000,000 search ad impressions from a
large search engine,\footnote{While we report results on a single
  dataset, we repeated the experiments on two others, producing
  qualitatively the same results.  No user-specific data was used in
  these experiments.} corresponding to ads shown on a small set of
search queries.  We formed examples with a feature vector $\theta_t$
for each ad impression, using features based on the text of the ad and
the query, as well as where on the page the ad showed.  The target
label $y_t$ is 1 if the ad was clicked, and -1 otherwise.

Smaller learning-rates worked better on this dataset; for each
(algorithm, $\lambda)$ combination we chose the best $\gamma$ from 9
points in the range $[0.03, 0.20]$.
Rather than shuffling, we report results for a single pass over the
data using the best $\gamma$, processing the events in the order the
queries actually occurred.  We also set a lower bound for the
stabilizing terms $\bs_t$ of 20.0, (corresponding to a maximum
learning rate of 0.05), as we found this improved accuracy somewhat.
Again, qualitative results did not depend on this choice.

\paragraph{Results}
Table~\ref{table} reports AUC accuracy (larger numbers are better),
followed by the density of the final predictor $x_T$ (number of
non-zeros divided by the total number of features present in the
training data).  We measured accuracy online, recording a prediction
for each example before training on it, and then computing the AUC for
this set of predictions.  For these experiments, we fixed $\lambda =
0.05/T$ (where $T$ is the number of examples in the dataset), which
was sufficient to introduce non-trivial sparsity.  Overall, there is
very little difference between the algorithms in terms of accuracy,
with RDA having a slight edge for these choices for $\lambda$.  
Our main point concerns the sparsity numbers.  It has been shown
before that RDA outperforms FOBOS in terms of sparsity.  The question
then is how does \prox perform, as it is a hybrid of the two,
selecting additional stabilization $R_t$ in the manner of \fobos, but
handling the $L_1$ regularization in the manner of \orgn.  These
results make it very clear: it is the treatment of $L_1$
regularization that makes the key difference for sparsity, as \prox
behaves very comparably to \orgn in this regard.

\begin{figure}[t!]
\begin{center}
\includegraphics[width=3.2in]{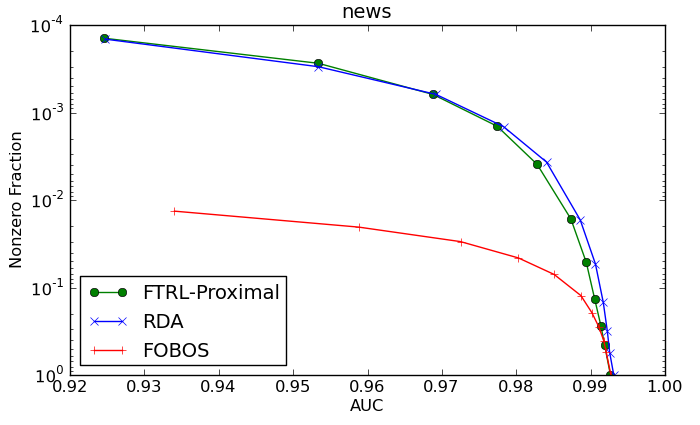}
\end{center}
\caption{ Sparsity versus accuracy tradeoffs on the 20 newsgroups
  dataset.  Sparsity increases on the y-axis, and AUC increases on the
  x-axis, so the top right corner gets the best of both worlds.
  \fobos is pareto-dominated by \prox and \orgn.}
\label{fig:pareto-news}
\end{figure}

\begin{figure}[t!]
\begin{center}
\includegraphics[width=3.2in]{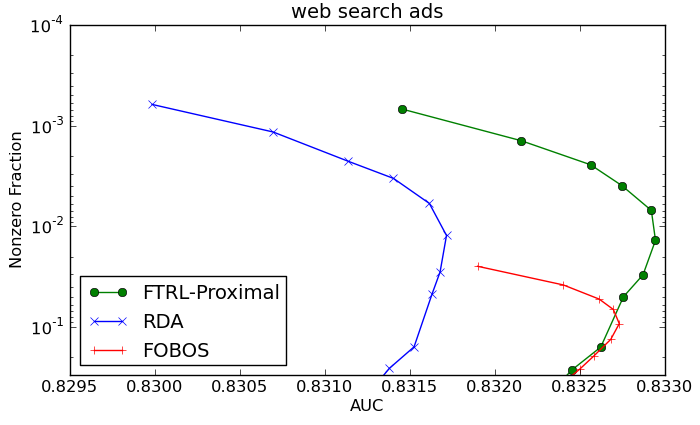}
\end{center}
\caption{ The same comparison as the previous figure, but on a large
  search ads ranking dataset.  On this dataset, \prox significantly
  outperforms both other algorithms.} \label{fig:pareto-ads}
\end{figure}

Fixing a particular value of $\lambda$, however, does not tell the
whole story.  For all these algorithms, one can trade off accuracy to
get more sparsity by increasing the $\lambda$ parameter.  The best
choice of this parameter depends on the application as well as the
dataset.  For example, if storing the model on an embedded device with
expensive memory, sparsity might be relatively more important.  To
show how these algorithms allow different tradeoffs, we plot sparsity
versus AUC for the different algorithms over a range of $\lambda$
values.  Figure~\ref{fig:pareto-news} shows the tradeoffs for the 20
newsgroups dataset, and Figure~\ref{fig:pareto-ads} shows the
tradeoffs for web search ads.

In all cases, \fobos is pareto-dominated by \orgn and \prox.  These
two algorithms are almost indistinguishable in the their tradeoff
curves on the newsgroups dataset, but on the ads dataset \prox
significantly outperforms \orgn as well.\footnote{The improvement is more
  significant than it first appears.  A simple model with only
  features based on where the ads were shown achieves an AUC of nearly
  0.80, and the inherent uncertainty in the clicks means that even
  predicting perfect probabilities would produce an AUC significantly
  less than 1.0, perhaps 0.85.}

\section{Conclusions and Open Questions}\label{sec:disc}

The goal of this work has been to extend the theoretical understanding
of several families of algorithms that have shown significant applied
success for large-scale learning problems.  We have shown that the
most commonly used versions of mirror descent, FTRL-Proximal and RDA
are closely related, and provided evidence that the non-smooth
regularization $\Psi$ is best handled globally, via RDA or
FTRL-Proximal.  Our analysis also extends these algorithms to implicit
updates, which can offer significantly improved performance for some
problems, including applications in active learning and
importance-weighted learning.

Significant open questions remain.  The observation that FOBOS
is using a subgradient approximation for much of the cumulative $L_1$
penalty while RDA and FTRL-Proximal handle it exactly provides a
compelling explanation for the improved sparsity produced by the
latter two algorithms.  Nevertheless, this is not a proof that these
two algorithms always produce more sparsity.  Quantitative bounds on
sparsity have proved theoretically very challenging, and any
additional results in this direction would be of great interest.

Similar challenges exist with quantifying the advantage offered by
implicit updates.  Our bounds demonstrate, essentially, a one-step
advantage for implicit updates: on any given update, the implicit
update will increase the regret bound by no more than the explicit
linearized update, and the inequality will be strict whenever the
implicit update is non-trivial.
However, this is insufficient to say that for any given learning
problem implicit updates will offer a better bound.  After one update,
the explicit and implicit algorithms will be at different feasible
points $x\ti$, which means that they will suffer different losses
under $f\ti$ and (more importantly) compute and store different
gradients for that function.

This issue is not unique to implicit updates: anytime the real loss
functions $f_t$ are non-linear, but the algorithm approximates them by
computing $g_t = \grad f_t(x_t)$, two different first-order algorithms
may see a different sequence of $g_t$'s; since tight regret bounds
depend on this sequence, the bounds will not be directly comparable.
Generally we assume the gradients are bounded, $\smnorm{g_t} \leq G$,
which leads to bounds like $\BO(G\sqrt{T})$, but since a large number
of algorithms obtain this bound, it cannot be used to discriminate
between them.  Developing finer-grained techniques that can accurately
compare the performance of different first-order online algorithms on
non-linear functions could be of great practical interest to the
learning community since the loss functions used are almost never
linear.

\section*{Acknowledgments}
The author wishes to thank Matt Streeter for numerous helpful
discussions and comments, and Fernando Pereira for a conversation that
helped focus this work on the choice $\Psi(x) =\norm{x}_1$.

\bibliography{equiv_and_implicit_long.bib}

\end{document}

com

%% file: equiv_and_implicit_arxiv.bbl
\begin{thebibliography}{28}
\providecommand{\natexlab}[1]{#1}
\providecommand{\url}[1]{\texttt{#1}}
\expandafter\ifx\csname urlstyle\endcsname\relax
  \providecommand{\doi}[1]{doi: #1}\else
  \providecommand{\doi}{doi: \begingroup \urlstyle{rm}\Url}\fi

\bibitem[Abernethy et~al.(2008)Abernethy, Bartlett, Rakhlin, and
  Tewari]{abernethy08}
Jacob Abernethy, Peter~L. Bartlett, Alexander Rakhlin, and Ambuj Tewari.
\newblock Optimal strategies and minimax lower bounds for online convex games.
\newblock In \emph{COLT}, 2008.

\bibitem[Bartlett et~al.(2007)Bartlett, Hazan, and Rakhlin]{bartlett07adaptive}
Peter~L. Bartlett, Elad Hazan, and Alexander Rakhlin.
\newblock Adaptive online gradient descent.
\newblock In \emph{NIPS}, 2007.

\bibitem[Beygelzimer et~al.(2010)Beygelzimer, Hsu, Langford, and
  Tong]{beygelzimer10agnostic}
Alina Beygelzimer, Daniel Hsu, John Langford, and Zhang Tong.
\newblock Agnostic active learning without constraints.
\newblock In \emph{NIPS}, 2010.

\bibitem[Boyd and Vandenberghe(2004)]{boyd}
Stephen Boyd and Lieven Vandenberghe.
\newblock \emph{Convex Optimization}.
\newblock Cambridge University Press, 2004.

\bibitem[Chang and Lin(2010)]{libsvmdata}
Chih-Chung Chang and Chih-Jen Lin.
\newblock {LIBSVM} data sets.
\newblock {\tiny
  \url{http://www.csie.ntu.edu.tw/~cjlin/libsvmtools/datasets/}}, 2010.

\bibitem[Do et~al.(2009)Do, Le, and Foo]{do09proximal}
Chuong~B. Do, Quoc~V. Le, and Chuan-Sheng Foo.
\newblock Proximal regularization for online and batch learning.
\newblock In \emph{ICML}, 2009.

\bibitem[Dredze(2010)]{sentiment}
Mark Dredze.
\newblock Multi-domain sentiment dataset (v2.0).
\newblock {\tiny \url{http://www.cs.jhu.edu/~mdredze/datasets/sentiment/}},
  2010.

\bibitem[Duchi and Singer(2009)]{duchi09fobos}
John Duchi and Yoram Singer.
\newblock Efficient learning using forward-backward splitting.
\newblock In \emph{NIPS}. 2009.

\bibitem[Duchi et~al.(2010{\natexlab{a}})Duchi, Hazan, and
  Singer]{duchi10adaptive}
John Duchi, Elad Hazan, and Yoram Singer.
\newblock Adaptive subgradient methods for online learning and stochastic
  optimization.
\newblock In \emph{COLT}, 2010{\natexlab{a}}.

\bibitem[Duchi et~al.(2010{\natexlab{b}})Duchi, Shalev-Shwartz, Singer, and
  Tewari]{duchi10composite}
John Duchi, Shai Shalev-Shwartz, Yoram Singer, and Ambuj Tewari.
\newblock Composite objective mirror descent.
\newblock In \emph{COLT}, 2010{\natexlab{b}}.

\bibitem[Hazan(2008)]{hazan08extract}
Elad Hazan.
\newblock Extracting certainty from uncertainty: Regret bounded by variation in
  costs.
\newblock In \emph{COLT}, 2008.

\bibitem[Kakade et~al.(2009)Kakade, Shalev-shwartz, and
  Tewari]{kakade09smoothness}
Sham~M. Kakade, Shai Shalev-shwartz, and Ambuj Tewari.
\newblock On the duality of strong convexity and strong smoothness: Learning
  applications and matrix regularization.
\newblock 2009.

\bibitem[Kalai and Vempala(2005)]{kalai03ftpl}
Adam Kalai and Santosh Vempala.
\newblock Efficient algorithms for online decision problems.
\newblock \emph{Journal of Computer and Systems Sciences}, 71\penalty0 (3),
  2005.

\bibitem[Karampatziakis and Langford(2010)]{karampatziakis10importance}
Nikos Karampatziakis and John Langford.
\newblock Importance weight aware gradient updates.
\newblock \url{http://arxiv.org/abs/1011.1576}, 2010.

\bibitem[Kivinen and Warmuth(1997)]{kivinen94exponentiated}
Jyrki Kivinen and Manfred Warmuth.
\newblock {Exponentiated Gradient Versus Gradient Descent for Linear
  Predictors}.
\newblock \emph{Journal of Information and Computation}, 132, 1997.

\bibitem[Kivinen et~al.(2006)Kivinen, Warmuth, and Hassibi]{kivinen06pnorm}
Jyrki Kivinen, Manfred Warmuth, and Babak Hassibi.
\newblock The p-norm generalization of the lms algorithm for adaptive
  filtering.
\newblock \emph{IEEE Transactions on Signal Processing}, 54(5), 2006.

\bibitem[Kulis and Bartlett(2010)]{kulis10implicit}
Brian Kulis and Peter Bartlett.
\newblock Implicit online learning.
\newblock In \emph{ICML}, 2010.

\bibitem[Lee et~al.(2006)Lee, Lee, Abbeel, and Ng]{lee06efficient}
Su-In Lee, Honglak Lee, Pieter Abbeel, and Andrew~Y. Ng.
\newblock Efficient l1 regularized logistic regression.
\newblock In \emph{AAAI}, 2006.

\bibitem[McMahan(2010)]{mcmahan10equiv}
H.~Brendan McMahan.
\newblock {Follow-the-Regularized-Leader and Mirror Descent:\\ Equivalence
  Theorems and L1 Regularization}.
\newblock Submitted, 2010.

\bibitem[McMahan and Streeter(2010)]{mcmahan10boundopt}
H.~Brendan McMahan and Matthew Streeter.
\newblock Adaptive bound optimization for online convex optimization.
\newblock In \emph{COLT}, 2010.

\bibitem[Shalev-Shwartz and Kakade(2008)]{shwartz08mind}
Shai Shalev-Shwartz and Sham~M. Kakade.
\newblock Mind the duality gap: Logarithmic regret algorithms for online
  optimization.
\newblock In \emph{NIPS}, pages 1457--1464, 2008.

\bibitem[Shalev-Shwartz and Singer(2006)]{shwartz06repeated}
Shai Shalev-Shwartz and Yoram Singer.
\newblock Convex repeated games and fenchel duality.
\newblock In \emph{NIPS}, 2006.

\bibitem[Streeter and McMahan(2010)]{streeter10conditioning}
Matthew~J. Streeter and H.~Brendan McMahan.
\newblock Less regret via online conditioning.
\newblock \url{http://arxiv.org/abs/1002.4862}, 2010.

\bibitem[Sugiyama et~al.(2008)Sugiyama, Suzuki, Nakajima, Kashima, B\"unau, and
  Kawanabe]{sugiyama08direct}
Masashi Sugiyama, Taiji Suzuki, Shinichi Nakajima, Hisashi Kashima, Paul
  B\"unau, and Motoaki Kawanabe.
\newblock Direct importance estimation for covariate shift adaptation.
\newblock \emph{Annals of the Institute of Statistical Mathematics},
  60\penalty0 (4), 2008.

\bibitem[Xiao(2009)]{xiao09dualaveraging}
Lin Xiao.
\newblock Dual averaging method for regularized stochastic learning and online
  optimization.
\newblock In \emph{NIPS}, 2009.

\bibitem[Xiao(2010)]{xiao10dualavgjmlr}
Lin Xiao.
\newblock Dual averaging methods for regularized stochastic learning and online
  optimization.
\newblock \emph{Journal of Machine Learning Research}, 11, 2010.

\bibitem[Zinkevich(2003)]{zinkevich03giga}
Martin Zinkevich.
\newblock Online convex programming and generalized infinitesimal gradient
  ascent.
\newblock In \emph{ICML}, 2003.

\bibitem[Zinkevich(2004)]{zinkevich04thesis}
Martin Zinkevich.
\newblock \emph{Theoretical guarantees for algorithms in multi-agent settings}.
\newblock PhD thesis, Pittsburgh, PA, USA, 2004.

\end{thebibliography}
